\documentclass[twoside,11pt]{article}
\usepackage{jair, theapa, rawfonts}
\usepackage{times}
\usepackage{graphicx}
\usepackage{latexsym}
\usepackage{amsmath}
\usepackage{amssymb}
\usepackage{amsthm}
\usepackage{algorithm}
\usepackage{algorithmic}
\providecommand{\doarxiv}{true}
\usepackage{ifthen}
\newboolean{isarxiv}
\setboolean{isarxiv}{\doarxiv} 
\ifthenelse{\boolean{isarxiv}}{%
\newcommand{\arxiv}[1]{#1}
\newcommand{\notarxiv}[1]{}
}{
\newcommand{\arxiv}[1]{}
\newcommand{\notarxiv}[1]{#1}
}
\newcommand{\arxivalt}[2]{\ifthenelse{\boolean{isarxiv}}{#1}{#2}}
\newcommand{\arxivaltr}[2]{\ifthenelse{\boolean{isarxiv}}{#2}{#1}}
\newcommand{\narxiv}[1]{\notarxiv{#1}}

\newcommand{\truncated}{curve-normalized}

\newcommand{\lovasz}{Lov\'asz}

\newtheorem{theorem}{Theorem}

\newtheorem{lemma}[theorem]{Lemma}

\newtheorem{corollary}[theorem]{Corollary}

\newcommand{\curvf}[1]{\ensuremath{\kappa_{#1}}}

\begin{document}

\title{A Unified Framework of Constrained Robust Submodular Optimization with Applications}

\author{\name Rishabh Iyer \email rishabh.iyer@utdallas.edu \\ \addr University of Texas at Dallas,
USA } 

\maketitle

\begin{abstract}
Robust optimization is becoming increasingly important in machine learning applications. In this paper, we study a unified framework of robust  submodular optimization. We study this problem both from a minimization and maximization perspective (previous work has only focused on variants of robust submodular maximization). We do this under a broad range of combinatorial constraints including cardinality, knapsack, matroid as well as graph based constraints such as cuts, paths, matchings and trees. Furthermore, we also study robust submodular minimization and maximization under multiple submodular upper and lower bound constraints.  We show that all these problems are motivated by important machine learning applications including robust data subset selection, robust co-operative cuts and robust co-operative matchings. In each case, we provide scalable approximation algorithms and also study hardness bounds. Finally, we empirically demonstrate the utility of our algorithms on synthetic data, and real world applications of robust cooperative matchings for image correspondence, robust data subset selection for speech recognition, and image collection summarization with multiple queries. \looseness-1
\end{abstract}

\section{Introduction}
\label{introduction}
Submodular functions provide a rich class of expressible models for a
variety of machine learning problems. They occur
naturally in two flavors. In minimization problems, they model
notions of cooperation, attractive potentials, and economies of scale,
while in maximization problems, they model aspects of coverage,
diversity, and information. A set function $f: 2^V \to \mathbb R$ over
a finite set $V = \{1, 2, \ldots, n\}$ is \emph{submodular} 
\cite{fujishige2005submodular} if for all
subsets $S, T \subseteq V$, it holds that 
\begin{align}
    f(S) + f(T) \geq f(S \cup
T) + f(S \cap T)
\end{align}
Given a set $S \subseteq V$, we define the
\emph{gain} of an element $j \notin S$ in the context $S$ as $f(j | S)
= f(S \cup j) - f(S)$. A perhaps more intuitive
characterization of submodularity is as follows:
a function $f$ is submodular if it satisfies
\emph{diminishing marginal returns}, namely 
$f(j | S) \geq f(j | T)$
for all $S \subseteq T, j \notin T$, and is \emph{monotone} if $f(j |
S) \geq 0$ for all $j \notin S, S \subseteq V$.

Two central optimization problems involving submodular functions are submodular minimization and submodular maximization. Moreover, it is often natural to want to optimize these functions subject to combinatorial constraints~\cite{nemhauser1978,rkiyersemiframework2013,jegelka2011-inference-gen-graph-cuts,goel2009optimal}. 

In this paper, we shall study the problem of robust submodular optimization. Often times in applications we want to optimize several objectives (or criteria) together. There are two natural formulations of this. One is the average case, where we can optimize the (weighted) sum of the submodular functions. Examples of this have been studied in data summarization applications~\cite{lin2012learning,tsciatchek14image,gygli2015video}. The other is robust or worst case, where we want to minimize (or maximize) the maximum (equivalently minimum) among the functions. Examples of this have been proposed for sensor placement and observation selection~\cite{krause08robust}. Robust or worst case optimization is becoming increasingly important since solutions achieved by minimization and maximization can be unstable to perturbations in data. Often times submodular functions in applications are instantiated from various properties of the data (features, similarity functions, clusterings etc.) and obtaining results which are robust to perturbations and variations in this data is critical.

Given monotone submodular functions $f_1, f_2, \cdots, f_l$ to minimize and $g_1, g_2, \cdots, g_k$ to maximize, consider the following problems:
\begin{align}
\label{mo-subopt}
\mbox{Problem 1: }\min_{X \in \mathcal C} \max_{i = 1:l} f_i(X), \,\,\,\,\,\,
\mbox{Problem 2: } \max_{X \in \mathcal C} \min_{i = 1:k} g_i(X) \nonumber
\end{align}
$\mathcal C$ stands for combinatorial constraints, which include cardinality, matroid, spanning trees, cuts, s-t paths etc. We shall call these problems \textsc{Robust-SubMin} and \textsc{Robust-SubMax}. Note that when $k = 1, l = 1$, we get back constrained submodular minimization and constrained submodular maximization. We will also study special cases of \textsc{Robust-SubMin} and \textsc{Robust-SubMax} when the constraints are defined via another submodular function. We study two problems: a) minimize the functions $f_i$'s while having lower bound constraints on the $g_i$'s (\textsc{Robust-scsc}), and b) maximize the functions $g_i$'s subject to upper bound constraints on $f_i$'s (\textsc{Robust-scsk}).  
\begin{align} 
    \mbox{Problem 3: } \min_{X \subseteq V}  \max_{i=1:l} f_i(X) 
    \,\, | \,\, g_i(X) \geq c_i, i = 1, \cdots k \nonumber \\
    \mbox{Problem 4: }\max_{X \subseteq V} \min_{i = 1:k} g_i(X)
    \,\, | \,\, f_i(X) \leq b_i, i = 1, \cdots l \nonumber
\end{align}
Problems 3 and 4 attempt to simultaneously minimize the functions $f_i$ while maximizing $g_i$. \arxiv{Finally, a natural extension of these problems is to have a joint average/worst case objective where we optimize $(1 - \lambda) \min_{i = 1:k} g_i(X) + \lambda/k \sum_{i = 1}^k g_i(X)$ and $(1 - \lambda) \max_{i = 1:l} f_i(X) + \lambda/l \sum_{i = 1}^l f_i(X)$ in Problems 1 - 4. We shall call these problems the \textsc{Mixed} versions (\textsc{Mixed-SubMax}, \textsc{Mixed-SubMin} etc.) However we point out that the \textsc{Mixed} case for Problems 1-4 is a special case of \textsc{Robust} optimization. Its easy to see that we can convert this into a \textsc{Robust} formulation by defining $f^{\prime}_i(X) = (1 - \lambda)f_i(X) + \lambda/l \sum_{i = 1}^l f_i(X)$ and $g^{\prime}_i(X) = (1 - \lambda)g_i(X) + \lambda/k \sum_{i = 1}^k g_i(X)$ and then defining the \textsc{Robust} optimization on $f^{\prime}_i$'s and $g^{\prime}_i$.}\looseness-1

\paragraph{Special Cases of Problems 3 and 4.} We get several special cases of \textsc{Robust-scsc} and \textsc{Robust-scsk}. When $l = k = 1$, Problems 3 and 4 generalize \textsc{Scsc} and \textsc{Scsk}~\cite{nipssubcons2013}. Similarly, when the functions $f_i$'s are modular, Problem 4 (\textsc{Robust-scsk}) becomes Robust submodular maximization (\textsc{Robust-SubMax}) under multiple knapsack constraints, which in turn generalizes \textsc{Robust-SubMax} subject to a single knapsack constraint (studied in~\cite{anari2019robust}), and single cardinality constraint (studied in~\cite{krause08robust}. When $f_i$'s are modular and $k = 1$, Problem 3 (\textsc{Robust-scsc}) becomes Robust Submodular Set Cover (\textsc{Robust-ssc}), and when $g_i$'s are modular, Problem 3 becomes robust submodular minimization (\textsc{Robust-SubMin}) under multiple knapsack constraints. Similarly, when $g_i$'s are modular, Problem 4 (\textsc{Robust-scsk}) becomes Robust Submodular Span (\textsc{Robust-ss}) problem.

\paragraph{Conference version of this paper:} This work significantly extends our conference paper~\cite{iyer2020robust}. The conference version just studied Problem 1, while in this work, we also study Problems 2, 3 and 4 and provide a unified perspective to the different instantiations of Robust Submodular Optimization. In particular, we provide algorithms for Problem 2 with multiple knapsack constraints, and introduce and study Problems 3 and 4 (which to the best of our knowledge are new problems) with multiple modular and submodular constraints. This paper also significantly extends upon the conference version from an empirical perspective, by considering new applications of robust data selection with vocabulary constraints, and query focused summarization with multiple queries/targets (see motivating applications below).

\subsection{Motivating Applications}
This section provides an overview of two specific applications which motivate \textsc{Robust-SubMin}. 
\begin{figure}
\begin{center}
\includegraphics[width=0.45\textwidth]{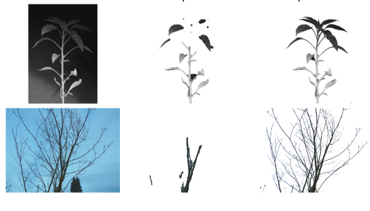}
~
\includegraphics[width=0.45\textwidth]{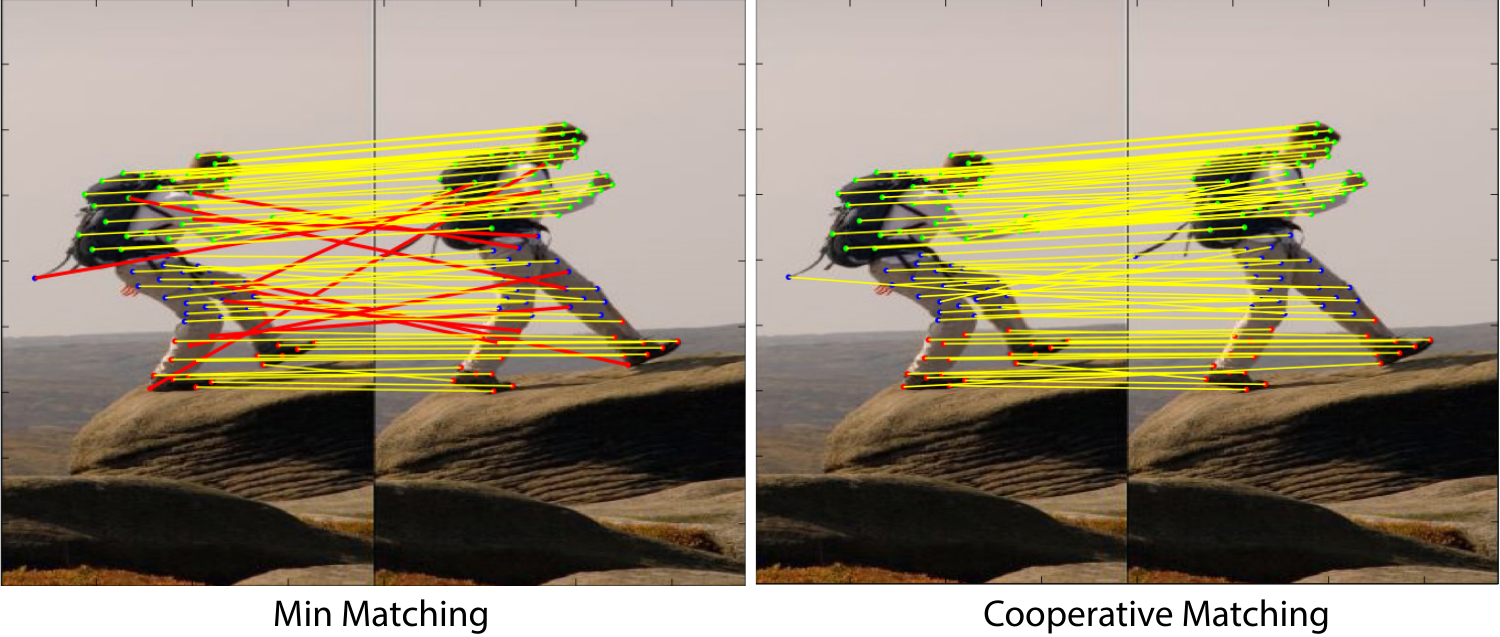}
\vspace{2ex}

\includegraphics[width=0.5\textwidth]{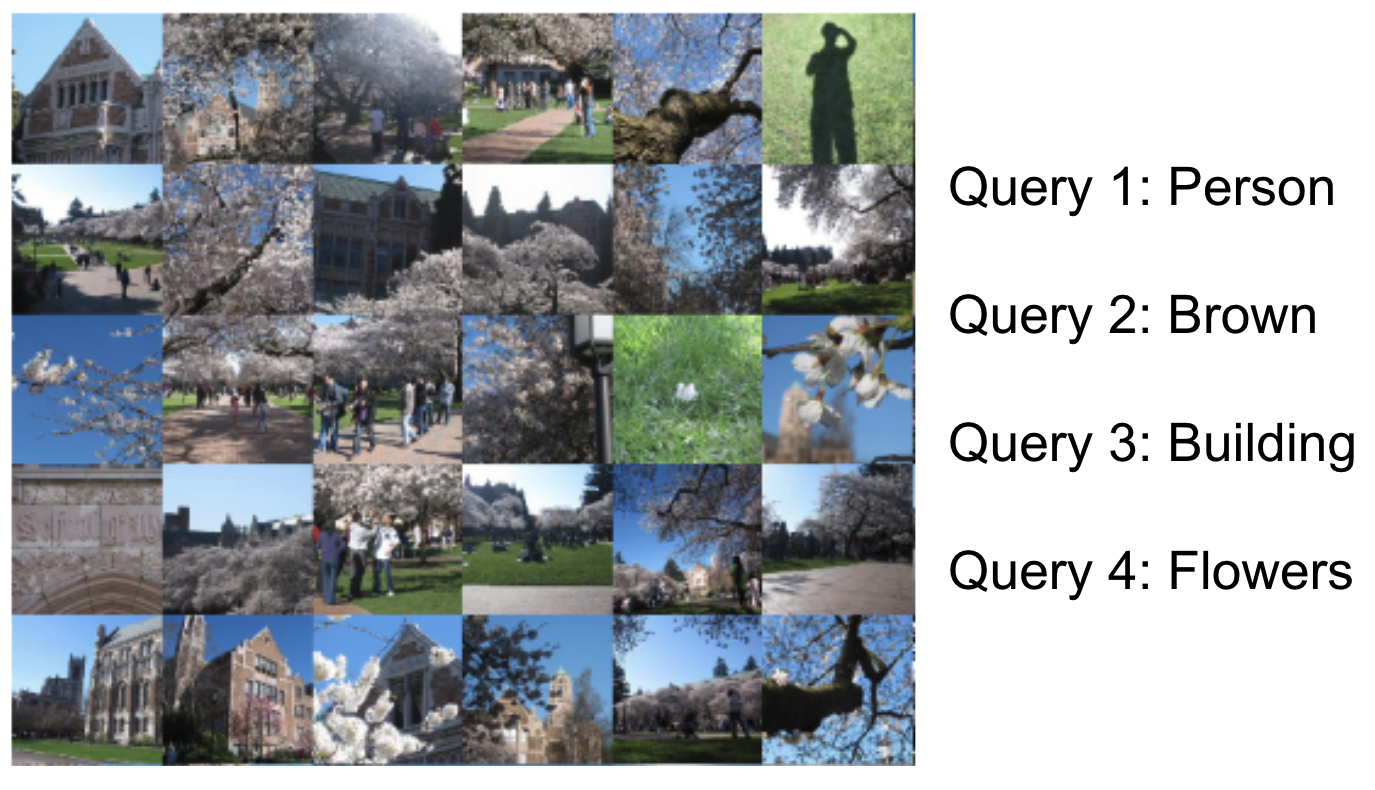}
\caption{
An illustration of (top left) co-operative cuts~\cite{jegelka2011-inference-gen-graph-cuts}, (top right) cooperative matchings~\cite{iyer2019near}, and (bottom) query focused summarization with multiple queries}
\end{center}
\label{fig:illustration}
\end{figure}

\paragraph{Robust Co-operative Cuts: } Markov random fields with pairwise attractive potentials occur naturally in modeling image segmentation and related applications~\cite{boykov2004experimental}. While models are tractably solved using graph-cuts, they suffer from the shrinking bias problem, and images with elongated edges are not segmented properly. When modeled via a submodular function, however, the cost of a cut is not just the sum of the edge weights, but a richer function that allows cooperation between edges, and yields superior results on many challenging tasks (see, for example, the results of the image segmentations in~\cite{jegelka2011-nonsubmod-vision}). This was achieved in~\cite{jegelka2011-nonsubmod-vision} by partitioning the set of edges $\mathcal E$ of the grid graph into groups of similar edges (or types) $\mathcal E_1, \cdots, \mathcal E_k$, and defining a function $f(S) = \sum_{i = 1}^k \psi_i(w(S \cap \mathcal E_i)), S \subseteq \mathcal E$, where $\psi_i$s are concave functions and $w$ encodes the edge potentials. This ensures that we offer a \emph{discount} to edges of the same type. The left image in Figure~1 illustrates co-operative cuts~\cite{jegelka2011-inference-gen-graph-cuts}. However, instead of taking a single clustering (i.e. a single group of edges $\mathcal E_1, \cdots, \mathcal E_k$), one can instantiate several clusterings  $\{(\mathcal E^1_1, \cdots, \mathcal E^1_k), \cdots, (\mathcal E^l_1, \cdots, \mathcal E^l_k)\}$ and define a robust objective: $f_{robust}(S) = \max_{i = 1:l} \sum_{j = 1}^k \psi_j(w(S \cap \mathcal E^i_j)), S \subseteq \mathcal E$. Minimizing $f_{robust}$ over the family of s-t cuts can achieve segmentations that are robust to many such groupings of edges thereby achieving \emph{robust co-operative segmentations}. We call this \emph{Robust Co-operative Cuts} and this becomes instance of \textsc{Robust-SubMin} with $f_{robust}$ defined above and the constraint $\mathcal C$ being the family of s-t cuts. Cooperative cuts is illustrated in Figure 1 (top-left).\looseness-1

\paragraph{Robust Co-operative Matchings: } The problem of matching key-points in images, also called image
correspondence, is an important problem in computer
vision~\cite{ogale2005shape}. The simplest model for this problem
constructs a matching with linear scores, i.e., a max bipartite
matching, called a \emph{linear
  assignment}. This model does not allow a representation of
interaction between the pixels, and as a result, we see many spurious matches in the results of min-matching (see top right group of plots in Figure 1). One way to address this is by handling interaction or cooperation between the keypoints~\cite{iyer2019near}, and one kind of desirable interaction is that similar or neighboring pixels be matched together. We can achieve this as follows. First we cluster the key-points in the two images into $k$ groups and the induced clustering of edges can be given a discount via a submodular function. Let $\{ V_i^{(1)} \}_{i=1}^k$ and $\{
V_i^{(2)} \}_{i=1}^k$ be the two sets of clusters.  We then compute
the linear assignment problem, letting $\mathcal M \subseteq \mathcal
E$ be the resulting maximum matching. We then partition the edge set $\mathcal
E = \mathcal E_1 \cup \mathcal E_2 \cup \dots \mathcal E_k \cup
\mathcal E'$ where $\mathcal E_i = \mathcal M \cap (V_\ell^{(1)}
\times V_s^{(2)} )$ for $\ell,s \in \{ 1, 2, \dots, k\}$ corresponding
to the $i$'th largest intersection, and $\mathcal E' = \{\mathcal E
\backslash \cup_{i = 1}^k \mathcal E_i\}$ are the remaining edges
either that were not matched or that did not lie within a frequently
associated pair of image key-point clusters. We then define a submodular function as follows: $f(S) = \sum_{i = 1}^k \psi_i(w(S \cap \mathcal E_i)) + 
w(S \cap \mathcal E')$, 
which provides an additional discount to the edges $\{\mathcal
E_i\}_{i = 1}^k$ corresponding to key-points that were frequently
associated in the initial pass. The difference between minimum matching and cooperative matching is shown in the top-right plots in Figure 1, and we show that through cooperation, we can effectively remove these spurious matchings. Similar to cooperative cuts, we can consider several clusterings, and try to solve a robust cooperative matching problem. To do this, we can consider several clusterings $\{(\mathcal E^1_1, \cdots, \mathcal E^1_k), \cdots, (\mathcal E^l_1, \cdots, \mathcal E^l_k)\}$, and define $f_i(S) = \sum_{j = 1}^k \psi_j(w(S \cap \mathcal E^i_j)) + 
w(S \cap \mathcal E^i{'})$. We then seem to minimize $f_{robust}(S) = \min_i f_i(S)$ subject to a matching constraint, and this becomes an instance of \textsc{Robust-SubMin}.\looseness-1

\paragraph{Robust Data Subset Selection: } Submodular functions have successfully been used for several data subset selection in domains such as image classification~\cite{kaushal2019learning}, speech recognition~\cite{wei2013using} and machine translation~\cite{kirchhoff2014-submodmt}. \cite{wei2015submodularity,killamsetty2020glister} prove that the problem of selecting the maximum likelihood subset of a training dataset is a submodular optimization problem for several classifiers and loss functions. Another approach, which we shall study in this paper, is selecting data-sets which are robust to several data subset selection models $g_1, \cdots, g_k$. These models can be defined via different submodular functions, different choice of features, perturbations in the feature space and different target distributions where we want to use these models. In this case, we can pose this as an instance of \textsc{Robust-SubMax} where we want to maximize the minimum among the utility functions Furthermore, it is also natural to select subsets of data which minimize the complexity of the dataset~\cite{liu2017svitchboard,lin2011optimal}. In this case, the functions $f$ captures the complexity of the data-sets $X$ (for example, vocabulary size in speech recognition and number of objects in object detection). This is naturally an instance of \textsc{Robust-SCSK} with $g_i$ being the data selection models, while $f$ is the complexity of the selected subset ($l = 1$). We can also define multiple complexity functions $f_i$, where each function is defined via different perturbations in the vocabulary (obtained by say randomly deleting a certain fraction of words from the vocabulary function).\looseness-1

\paragraph{Data Subset Selection and Data Summarization with Multiple Targets/Queries: } In many applications, it makes sense to select a subset (or summary) $X \subseteq V$, which is both diverse but also relevant to a given target or query $Q$~\cite{kaushal2020unified,iyer2020submodular,vasudevan2017query,li2012multi}. One way to frame such a problem is to maximize a diversity function $g(X)$ while simultaneously minimizing the conditional gain $h(X | Q)$~\cite{iyer2020submodular,kaushal2020unified}, where $h$ is an appropriate submodular function for the query relevance. Note that $h$ and $g$ could also be the same function. Two natural formulations of this are $\max_{X \subseteq V} g(X) \,\,\ | \,\,\ h(X | Q) \leq \epsilon$ or $\min_{X \subseteq V} h(X | Q) \,\,\ | \,\,\ g(X) \geq c$ -- these problems are instances of \textsc{Scsc} and \textsc{Scsk} respectively~\cite{nipssubcons2013}. A simple and realistic extension of this is where we have multiple query sets (equivalently targets) $Q_1, \cdots, Q_l$. Essentially this means we want to maximize $g(X)$ while minimizing $h(X | Q_i$ for $i = 1, \cdots, l$. In other words, we want to select a diverse summary while yet being \emph{equally} relevant to all the $l$ queries. Two natural formulations are: $\max_{X \subseteq V} g(X) \,\,\ | \,\,\ h(X | Q_i) \leq \epsilon_i, i = 1:l,$ and $\min_{X \subseteq V} \max_{i = 1:l} h(X | Q_i) \,\,\ | \,\,\ g(X) \geq c$. These are instances of \textsc{Robust-Scsk} and \textsc{Robust-Scsc} with $k = 1$ and $l$ the number of queries. In our experiments, we study the \textsc{Robust-Scsc} variant, and show that the robust approach ensures fairness among the queries (i.e. the summary is equally relevant of each query) much better compared to just taking a sum of the conditional functions ($\sum_{i = 1}^l h(X | Q_i)$ or the union conditional function ($ h(X | \cup_{i = 1}^l Q_i)$. The application of query focused summarization with multiple queries is illustrated in Figure 1 (right), where the goal is to select a summary which is relevant to all four queries simultaneously. Another related application of this is targeted data subset selection for machine learning, where the goal is to select among a set of unlabeled data points, a set of points related to (possibly multiple) targets~\cite{kaushal2021prism}. Examples of targets could be slices where the model is currently not performing well (e.g. in a self driving car application, the model may not be performing well to detect people at night, or vehicles in rainy and foggy conditions). It is also realistic in such scenarios that we have multiple targets/queries and the goal is to select a subset of unlabeled data while being relevant to all the target slices.

\subsection{Related Work}

\noindent \textbf{Submodular Minimization and Maximization: } Most Constrained forms of submodular optimization (minimization and maximization) are NP hard even when $f$ is monotone~\cite{nemhauser1978,wolsey1982analysis,goel2009optimal,jegelka2011-inference-gen-graph-cuts}. The greedy algorithm achieves a $1- 1/e$ approximation for cardinality constrained maximization and a $1/2$ approximation under matroid constraints~\cite{fisher1978analysis}. \cite{chekuri2011submodular} achieved a tight $1 - 1/e$ approximation for matroid constraints using the continuous greedy. \cite{kulik2009maximizing} later provided a similar $1 - 1/e$ approximation under multiple knapsack constraints. Constrained submodular minimization is much harder -- even with simple constraints such as a cardinality lower
bound constraints, the problems are not approximable better than a
polynomial factor of $\Omega(\sqrt{n} )$~\cite{svitkina2008submodular}. Similarly the problems of minimizing a submodular function under covering constraints~\cite{iwata2009submodular}, spanning trees, perfect matchings, paths~\cite{goel2009approximability} and cuts~\cite{jegelka2011-inference-gen-graph-cuts} have similar polynomial hardness factors. In all of these cases, matching upper bounds (i.e approximation algorithms) have been provided~\cite{svitkina2008submodular,iwata2009submodular,goel2009approximability,jegelka2011-inference-gen-graph-cuts}. In~\cite{rkiyersemiframework2013,curvaturemin}, the authors provide a scalable semi-gradient based framework and curvature based bounds which improve upon the worst case polynomial factors for functions with bounded curvature $\kappa_f$ (which several submodular functions occurring in real world applications have).\looseness-1 

\noindent \textbf{Robust Submodular Optimization: } Among the four problems we consider in this work, \textsc{Robust-SubMax} has been extensively studied in literature. One of the first papers to study robust submodular maximization was~\cite{krause08robust}, where the authors study \textsc{Robust-SubMax} with cardinality constraints. The authors reduce this problem to a submodular set cover problem using the \emph{saturate} trick to provide a bi-criteria approximation guarantee. \cite{anari2019robust} extend this work and study \textsc{Robust-SubMax} subject to matroid constraint. They provide bi-criteria algorithms by creating a union of $O(\log l/\epsilon)$ independent sets, with the union set having a guarantee of $1 - \epsilon$. They also discuss extensions to knapsack and multiple matroid constraints and provide bicriteria approximation of $(1 - \epsilon, O(\log l/\epsilon))$. \cite{powers2017constrained} also study the same problem. However, they take a different approach by presenting a bi-criteria algorithm that outputs a feasible set that is good only for a fraction of the $k$ submodular functions $g_i$. \cite{chen2017robust,wilder2017equilibrium} study a slightly general problem of robust non-convex optimization (of which robust submodular optimization is a special case), but they provide weaker guarantees compared to \cite{krause2008efficient,anari2019robust}. We introduced and studied \textsc{Robust-SubMin} in the conference version of this paper~\cite{iyer2020robust}, and \textsc{Robust-scsc} and \textsc{Robust-scsk} are new problems we study in this work. 

\noindent \textbf{Robust Min-Max Combinatorial Optimization: } From a minimization perspective, several researchers have studied robust min-max combinatorial optimization (a special case of \textsc{Robust-SubMin} with modular functions) under different combinatorial constraints (see~\cite{aissi2009min,kasperski2016robust} for a survey). Unfortunately these problems are NP hard even for constraints such as knapsack, s-t cuts, s-t paths, assignments and spanning trees where the standard linear cost problems are poly-time solvable~\cite{aissi2009min,kasperski2016robust}. Moreover, the lower bounds on hardness of approximation is $\Omega(log^{1-\epsilon} l)$ ($l$ is the number of functions) for s-t cuts, paths and assignments~\cite{kasperski2009approximability} and $\Omega(log^{1-\epsilon} n)$ for spanning trees~\cite{kasperski2011approximability} for any $\epsilon > 0$. For the case when $l$ is bounded (a constant), fully polynomial time approximation schemes have been proposed for a large class of constraints including s-t paths, knapsack, assignments and spanning trees~\cite{aissi2010general,aissi2009min,kasperski2016robust}. From an approximation algorithm perspective, the best known general result is an approximation factor of $l$ for constraints where the linear function can be exactly optimized in poly-time. For special cases such as spanning trees and shortest paths, one can achieve improved approximation factors of $O(\log n)$~\cite{kasperski2011approximability,kasperski2018approximating} and $\tilde{O}(\sqrt{n})$\footnote{Ignores $\log$ factors} ~\cite{kasperski2018approximating} respectively.\looseness-1

\subsection{Our Contributions}
While past work has mainly focused on \textsc{Robust-SubMax} (Problem 2), This paper is the first work to provide approximation bounds for \textsc{Robust-SubMin}, \textsc{Robust-SCSC}, and \textsc{Robust-SCSK}. Next, we enumerate our contribution for each of the four problems.
\begin{itemize}
    \item \textsc{Robust-SubMin: } We start by first providing the hardness results which follows from the hardness bounds of constrained submodular minimization and robust min-max combinatorial optimization. Next, we provide four families of algorithms for \textsc{Robust-SubMin}. The first and simplest approach approximates the $\max$ among the functions $f_1, \cdots, f_l$ in \textsc{Robust-SubMin} with the average. This in turn converts \textsc{Robust-SubMin} into a constrained submodular minimization problem. We show that solving the constrained submodular minimization with the average of the functions still yields an approximation factor for \textsc{Robust-SubMin} though it is worse compared to the hardness by a factor of $l$. While this approach is conceptually very simple, we do not expect this perform well in practice. We next study the Majorization-Minimization (\textsc{mmin}) family of algorithms, where we sequentially approximate the functions $f_1, \cdots, f_l$ with their modular upper bounds. Each resulting sub-problem involves solving a min-max combinatorial optimization problem and we use specialized solvers for the different constraints. The third algorithm is the Ellipsoidal Approximation (\textsc{ea}) where we replace the functions $f_i$'s with their ellipsoidal approximations. We show that this achieves the tightest approximation factor for several constraints, but is slower to the majorization-minimization approach in practice. The fourth technique is a continuous relaxation approach where we relax the discrete problem into a continuous one via the \lovasz{} extension. We show that the resulting robust problem becomes a convex optimization problem with an appropriate recast of the constraints. We then provide a rounding scheme thereby resulting in approximation factors for \textsc{Robust-SubMin}. Tables 1 and 2 show the hardness and the resulting approximation factors for several important combinatorial constraints. Along the way, we also provide a family of approximation algorithms and heuristics for \textsc{Robust-Min} (i.e. with modular functions $f_i$ under several combinatorial constraints). We finally compare the four algorithms for different constraints and in each case, discuss heuristics which will make the resulting algorithms practical and scalable for real world problems. 
    \item \textsc{Robust-SubMax: } For \textsc{Robust-SubMax}, we complement previous work by providing an approximation guarantee for multiple knapsack constraints. Previous works have only focused on cardinality constraints~\cite{krause08robust}, knapsack constraints and matroid constraints~\cite{anari2019robust}. 
    \item \textsc{Robust-SCSC} and \textsc{Robust-SCSK}: 
    We show that \textsc{Robust-SCSK} and \textsc{Robust-SCSC} are closely related in that given an approximation algorithm for one of the problems, we can achieve approximation algorithms for the other. Next, we provide a framework for approximation algorithms for these problems and in turn analyze the results for special case. Significantly, our results provide the first set of approximation algorithms for robust submodular maximization under multiple knapsack constraints, robust submodular minimization under single or multiple knapsack constraints and robust submodular set cover. Next, we provide a clear picture of the hardness results for these problems.
    \item Finally, we empirically show the performance of our algorithms on synthetic and real world datasets. We demonstrate the utility of our models in robust co-operative matching and show that our robust model outperforms simple co-operative model from~\cite{iyer2019near} for this task. We also demonstrate the utility of our algorithms for Robust SCSC and Robust SCSK, specifically for robust data subset selection.
\end{itemize}

\section{Preliminaries}
In this section, we will review some of the constructs and techniques used in this paper to provide approximation algorithms for \textsc{Robust-SubMin}.

\noindent \textbf{Curvature: } Given a submodular function $f$, the curvature~\cite{vondrak2010submodularity,conforti1984submodular,curvaturemin}: 
\begin{align*}
\kappa_f = 1 - \min_{j \in V} \frac{f(j | V \backslash j)}{f(j)}
\end{align*}
Thanks to submodularity, it is easy to see that $0 \leq \kappa_f \leq 1$. We define a quantity:
\begin{align}
    K(f, \kappa) = \frac{f}{1 + (1 - \kappa)(f - 1)}
\end{align}
where $\kappa$ is the curvature. Note that $K(f, \kappa)$ interpolates between $K(f, 1) = f$ and $K(f, 0) = 1$. This quantity shall repeatedly come up in the approximation and hardness bounds in this paper.

\noindent \textbf{The Submodular Polyhedron and \lovasz{} extension }
For a submodular function $f$, the submodular polyhedron $\mathcal P_f$ and the corresponding base polytope $\mathcal B_f$ are respectively defined as 
\begin{align}
    \mathcal P_f = \{ x : x(S) \leq f(S), \forall S \subseteq V \} 
\;\;\;
\mathcal B_f = \mathcal P_f \cap \{ x : x(V) = f(V) \}
\end{align}
For a vector $x \in \mathbb{R}^V$ and a set $X \subseteq V$, we write $x(X)
= \sum_{j \in X} x(j)$. Though $\mathcal P_f$ is defined via $2^n$ inequalities, its extreme point can be easily characterized~\cite{fujishige2005submodular,edmondspolyhedra}. Given any permutation $\sigma$ of the ground set $\{1, 2, \cdots, n\}$, and an associated chain $\emptyset = S^{\sigma}_0 \subseteq S^{\sigma}_1 \subseteq \cdots \subseteq S^{\sigma}_n = V$ with $S^{\sigma}_i =
\{ \sigma(1), \sigma(2), \dots, \sigma(i) \}$, a vector $h^f_{\sigma}$ satisfying,
\begin{align}
    h^f_{\sigma}(\sigma(i) = f(S^{\sigma}_i) - f(S^{\sigma}_{i-1}) = f(\sigma(i) | S^{\sigma}_{i-1}), \forall i = 1, \cdots, n
\end{align}
forms an extreme point of $\mathcal P_f$. Moreover, a natural convex extension of a submodular function, called the \lovasz{} extension~\cite{lovasz1983,edmondspolyhedra} is closely related to the submodular polyhedron, and is defined as $\hat{f}(x) = \max_{h \in \mathcal P_f} \langle h, x \rangle$. Thanks to the properties of the polyhedron, $\hat{f}(x)$ can be efficiently computed: Denote $\sigma_x$ as an ordering induced by $x$, such that $x(\sigma_x(1)) \geq x(\sigma_x(2)) \geq \cdots x(\sigma_x(n))$. Then the \lovasz{} extension is $\hat{f}(x) = \langle h^f_{\sigma_x}, x \rangle$~\cite{lovasz1983,edmondspolyhedra}. The gradient of the \lovasz{} extension $\nabla \hat{f}(x) = h^f_{\sigma_x}$.

\noindent \textbf{Modular lower bounds (Sub-gradients): } Akin to convex functions, submodular functions have tight modular lower bounds. These bounds are related to the sub-differential $\partial_f(Y)$ of the submodular set function $f$ at a set $Y \subseteq V$, which is defined 
\cite{fujishige2005submodular,edmondspolyhedra,iyer2015polyhedral}
as: 
\begin{align}
    \partial_f(Y) = \{y \in \mathbb{R}^n: f(X) - y(X) \geq f(Y) - y(Y),\; \text{for all } X \subseteq V\}
\end{align}
Denote a sub-gradient at $Y$ by $h_Y \in \partial_f(Y)$. Define $h_Y = h^f_{\sigma_Y}$ (see the definition of $h^f_{\sigma}$ from the previous paragraph) forms a lower bound of $f$, tight at $Y$ --- i.e.,
$h_Y(X) = \sum_{j \in X} h_Y(j) \leq f(X), \forall X
\subseteq V$ and $h_Y(Y) = f(Y)$. Notice that the extreme points of a sub-differential are a subset of the extreme points of the submodular polyhedron. 

\noindent \textbf{Modular upper bounds (Super-gradients): }
We can also define super-differentials $\partial^f(Y)$ of a submodular
function 
\cite{jegelka2011-nonsubmod-vision,rkiyersubmodBregman2012,rkiyersemiframework2013,iyer2015polyhedral}
at
$Y$: 
\begin{align}
    \partial^f(Y) = \{y \in \mathbb{R}^n: f(X) - y(X) \leq f(Y) - y(Y); \text{for all } X \subseteq V\}
\end{align}
It is possible, moreover, to provide specific super-gradients~\cite{rkiyersubmodBregman2012,rkiyersemiframework2013,iyer2015polyhedral,iyermirrordescent,iyer2015submodular} that define the following two modular upper bounds:
\begin{align}
m^f_{X, 1}(Y) \triangleq f(X) - \sum_{j \in X \backslash Y } f(j| X \backslash j) + \sum_{j \in Y \backslash X} f(j| \emptyset), \\
m^f_{X, 2}(Y) \triangleq f(X) - \sum_{j \in X \backslash Y } f(j| V \backslash j) + \sum_{j \in Y \backslash X} f(j| X).
\end{align}

Then $m^f_{X, 1}(Y) \geq f(Y)$ and $m^f_{X, 2}(Y) \geq f(Y), \forall Y \subseteq V$ and $m^f_{X, 1}(X) = m^f_{X, 2}(X) = f(X)$. Also note that $m^f_{\emptyset, 1}(Y) = m^f_{\emptyset, 2}(Y) = \sum_{j \in Y} f(j | \emptyset)$. For simplicity denote this as $m^f_{\emptyset}(X)$. Then the following result holds:
\begin{lemma}\cite{curvaturemin,rkiyersemiframework2013}
\label{curvaturemmin}
Given a submodular function $f$ with curvature $\kappa_f$, $f(X) \leq m^f_{\emptyset}(X) \leq K(|X|, \kappa_f) f(X)$
\end{lemma}
\noindent \textbf{Ellipsoidal Approximation: } Another generic approximation of a submodular function, introduced by Goemans et.\ al~\cite{goemans2009approximating}, is based on approximating the submodular polyhedron by an ellipsoid. The main result states that any polymatroid (monotone submodular) function $f$,
can be approximated by a function of the form $\sqrt{w^f(X)}$ for a
certain modular weight vector $w^f \in \mathbb R^V$, such that $\sqrt{w^f(X)} \leq f(X) \leq
O(\sqrt{n}\log{n}) \sqrt{w^f(X)}, \forall X \subseteq V$. A simple trick then provides a curvature-dependent approximation~\cite{curvaturemin} ---
we define the $\curvf{f}$-\emph{\truncated{}} version of $f$ as
follows: $f^{\kappa}(X) \triangleq \bigl[f(X) - {(1 - \curvf{f})} \sum_{j \in X} f(j)\bigr]/ \curvf{f}$. Define the function:
\begin{align}
    f^{\text{ea}}(X) = \curvf{f} \sqrt{w^{f^{\kappa}}(X)} + (1 -
  \curvf{f})\sum_{j \in X} f(j)
\end{align}

The following Lemma shows that $f^{\text{ea}}$ approximates $f$.
\begin{lemma}\label{curvatureea}
Given a submodular function $f$ with curvature $\kappa_f$, $f^{\text{ea}}$ satisfies $f^{\text{ea}}(X) \leq f(X) \leq K(\sqrt{n}, \kappa_f) f^{\text{ea}}(X), \forall X \subseteq V$.
\end{lemma}

\begin{table*}
       \begin{center}
       \tiny{
           \begin{tabular}{|l|c|c|c|c|c|c|}
  \hline
    Constraint & Hardness & \textsc{mmin-aa} & \textsc{ea-aa} & \textsc{mmin} & \textsc{ea} & CR
        \\ \hline
    Cardinality ($k$) & $K(\sqrt{n}, \kappa)$ & $l K(k, \kappa)$ & $l K(\sqrt{n}\log n, \kappa)$ & $O(\log l K(k, \kappa_{wc})/\log\log l)$ & $\tilde{O}(\sqrt{m \log l/\log \log l})$ & $n-k+1$ \\
    Trees & $M(K(n, \kappa), \log n)$ & $l K(n, \kappa)$ & $l K(\sqrt{m}\log m, \kappa)$ & $O(\min(\log n, l) K(n, \kappa_{wc}))$ &  $O(\min(\sqrt{\log n}, \sqrt{l}) \sqrt{m}\log m)$ & $m-n+1$ \\
    Matching & $M(K(n, \kappa), \log l)$ & $l K(n, \kappa)$ & $l K(\sqrt{m}\log m, \kappa)$ & $l K(n, \kappa)$ & $\tilde{O}(\sqrt{lm})$ & $n$ \\
    s-t Cuts & $M(K(\sqrt{m}, \kappa), \log l)$ & $l K(n, \kappa)$ & $l K(\sqrt{m}\log m, \kappa)$ & $l K(n, \kappa)$ & $\tilde{O}(\sqrt{lm})$ & $n$ \\
    s-t Paths & $M(K(\sqrt{m}, \kappa), \log l)$ & $l K(n, \kappa)$ & $l K(\sqrt{m}\log m, \kappa)$ & $O(\min(\sqrt{n}, l) K(n, \kappa_{wc}))$ & $O(\min(n^{0.25}, \sqrt{l})\sqrt{m}\log m)$ & $m$ \\
    Edge Cov. & $K(n, \kappa)$ & $l K(n, \kappa)$ & $l K(\sqrt{m}\log m, \kappa)$ & $l K(n, \kappa)$ & $\tilde{O}(\sqrt{lm})$ & $n$ \\
    Vertex Cov. & $2$ & $l K(n, \kappa)$ & $l K(\sqrt{n}\log n, \kappa)$ & $l K(n, \kappa)$ & $\tilde{O}(\sqrt{ln})$ & $2$ \\
   \hline
  \end{tabular}
    \label{tab:1}
    \caption{Approximation bounds and Hardness for \textsc{Robust-SubMin}. $M(.)$ stands for $\max(.)$}
    }
\end{center}
\end{table*}

\begin{table*}
\begin{center}
\scriptsize{
\begin{tabular}{|l|c|c|c|}
  \hline
    Constraint & Hardness & \textsc{mmin} & \textsc{ea}
        \\ \hline
    Knapsack & $K(\sqrt{n}, \kappa)$ & $K(n, \kappa)$ & $O(K(\sqrt{n}\log n, \kappa))$ \\
    Trees & $K(n, \kappa)$ & $K(n, \kappa)$ & $O(K(\sqrt{m}\log m, \kappa))$ \\
    Matchings & $K(n, \kappa)$ & $K(n, \kappa)$ & $O(K(\sqrt{m}\log m, \kappa))$ \\
    s-t Paths & $K(n, \kappa)$ & $K(n, \kappa)$ & $O(K(\sqrt{m}\log m, \kappa))$ \\
   \hline
  \end{tabular}}
    \label{tab:2}
    \caption{Approximation bounds and Hardness of in \textsc{Robust-SubMin} with $l$ constant}
\end{center}
\end{table*}

\section{\textsc{Robust-SubMin}}
In this section, we shall go over the hardness and approximation algorithms for \textsc{Robust-SubMin}. Recall that the \textsc{Robust-SubMin} optimization problem is:
\begin{align*}
    \mbox{Problem 1: } \min_{X \in \mathcal C} \max_{i = 1:l} f_i(X)
\end{align*}
We shall consider two cases, one where $l$ is bounded (i.e. its a constant), and the other where $l$ is unbounded. We start with some notation. We denote the graph as $G = (\mathcal V, \mathcal E)$ with $|\mathcal V| = n, |\mathcal E| = m$.  Depending on the problem at hand, the ground set $V$ can either be the set of edges ($V = \mathcal E$) or the set of vertices ($V = \mathcal V$). The groundset is the set of edges in the case of trees, matchings, cuts, paths and edge covers, while in the case of vertex covers, they are defined on the vertices. $\mathcal C$ is the combinatorial constraints which enforces that the set of edges/vertices is either an s-t cut, s-t path, vertex cover, edge cover, spanning tree, or matching etc. Tables 1 and 2 show the hardness and the approximation factors for several important combinatorial constraints and settings.

\subsection{Hardness of \textsc{Robust-SubMin}} 
Since \textsc{Robust-SubMin} generalizes robust min-max combinatorial optimization (when the functions are modular), we have the hardness bounds from~\cite{kasperski2009approximability,kasperski2011approximability}. For the modular case, the lower bounds are $\Omega(log^{1-\epsilon} l)$ ($l$ is the number of functions) for s-t cuts, paths and assignments~\cite{kasperski2009approximability} and $\Omega(log^{1-\epsilon} n)$ for spanning trees~\cite{kasperski2011approximability} for any $\epsilon > 0$. These hold unless $NP \subseteq DTIME(n^{poly \log n})$\cite{kasperski2009approximability,kasperski2011approximability}. Moreover, since \textsc{Robust-SubMin} also generalizes constrained submodular minimization, we get the curvature based hardness bounds from~\cite{curvaturemin,goel2009optimal,jegelka2011-inference-gen-graph-cuts}. The hardness results are in the first column of Table~1. The curvature $\kappa$ corresponds to the worst curvature among the functions $f_i$ (i.e. $\kappa = \max_i \kappa_i$). 

\subsection{Algorithms with Modular functions $f_i$'s}
In this section, we shall study approximation algorithms for \textsc{Robust-SubMin} when the functions $f_i$'s are modular, i.e. $f_i(X) = \sum_{j \in X} f_i(j)$. We call this problem \textsc{Robust-Min}. Most previous work~\cite{aissi2010general,aissi2009min,kasperski2016robust} has focused on fully polynomial approximation schemes when $l$ is small. These algorithms are exponential in $l$ and do not scale beyond $l = 3$ or $4$. Instead, we shall study approximation algorithms for this problem. Define two simple approximations of the function $F(X) = \max_i f_i(X)$ when $f_i$'s are modular. The first is $\hat{F}(X) = \sum_{i \in X} \max_{j = 1:l} f_j(i)$ and the second is $\tilde{F}(X) = 1/l \sum_{i = 1}^l \sum_{j \in X} f_i(j) = 1/l \sum_{i = 1}^l f_i(X)$. 

 \begin{lemma}\label{mod-approximations}
 Given $f_i(j) \geq 0$, it holds that $\hat{F}(X) \geq F(X) \geq \frac{1}{l} \hat{F}(X)$. Furthermore, $\tilde{F}(X) \leq F(X) \leq l\tilde{F}(X)$. Given a $\beta$-approximate algorithm for optimizing linear cost functions over the constraint $\mathcal C$, we can achieve an $l\beta$-approximation algorithm for \textsc{Robust-Min}.\looseness-1
 \end{lemma}
 \begin{proof}
We first prove the second part.  Its easy to see that $F(X) = \max_i f_i(X) \leq \sum_i f_i(X) = l\tilde{F}(X)$ which is the second inequality. The first inequality follows from the fact that $\sum_i f_i(X) \leq l \max_i f_i(X)$. To prove the first result, we start with proving $F(X) \leq \hat{F}(X)$. For a given set $X$, let $i_X$ be the index which maximizes $F$ so $F(X) = \sum_{j \in X} f_{i_X}(j)$. Then $f_{i_X}(j) \leq \max_i f_i(j)$ from which we get the result. Next, observe that $\hat{F}(X) \leq \sum_{i = 1}^l f_i(X) = l\tilde{F}(X) \leq lF(X)$ which follows from the inequality corresponding to $\tilde{F}$. 

Now given a $\beta$-approximation algorithm for optimizing linear cost functions over the constraint $\mathcal C$, denote $\tilde{X}$ by optimizing $\tilde{F}(X)$ over $\mathcal C$. Also denote $X^*$ as the optimal solution by optimizing $F$ over $\mathcal C$ and $\tilde{X^*}$ be the optimal solution
for optimizing $\tilde{F}$ over $\mathcal C$. Since $\tilde{X}$ is a $\beta$ approximation for $\tilde{F}$ over $\mathcal C$, $\tilde{F}(\tilde{X}) \leq \beta \tilde{F}(\tilde{X^*}) \leq \beta \tilde{F}(X^*)$. The last inequality holds since $X^*$ is feasible (i.e. it belongs to $\mathcal C$) so it must hold that $\tilde{F}(X^*) \geq \tilde{F}(\tilde{X^*})$.
A symmetric argument applies to $\hat{F}$. In particular, $F(\hat{X}) \leq \hat{F}(\hat{X}) \leq \beta \hat{F}(X^*) \leq \beta l F(X^*)$. The inequality $\hat{F}(\hat{X}) \leq \beta \hat{F}(X^*)$ again holds due to an argument similar to the $\tilde{F}$ case.
 \end{proof}
 
 Note that for most constraints $\mathcal C$, $\beta$ is $1$. In practice, we can take the better of the two solutions obtained from the two approximations above.

Next, we shall look at higher order approximations of $F$. Define the function $F_a(X) = (\sum_{i = 1}^l [f_i(X)]^a)^\frac{1}{a}$. Its easy to see that $F_a(X)$ comes close to $F$ as $a$ becomes large. 
\begin{lemma}
Define $F_a(X) = (\sum_{i = 1}^l [f_i(X)]^a)^\frac{1}{a}$. Then it holds that $F(X) \leq F_a(X) \leq l^\frac{1}{a} F(X)$. Moreover, given a $\beta$-approximation algorithm for optimizing $F_a(X)$ over $\mathcal C$, we can obtain a $\beta l^\frac{1}{a}$-approximation for \textsc{Robust-Min}.
\end{lemma}
\arxiv{
\begin{proof}
Its easy to see that $F(X) \leq F_a(X)$ since $[F(X)]^a \leq \sum_{i = 1}^l [f_i(X)]^a$. Next, note that $\forall i, f_i(X) \leq F(X)$ and hence $\sum_{i = 1}^l [f_i(X)]^a \leq l [F(X)]^a$ which implies that $F_a(X) \leq k^{\frac{1}{a}} F(X)$ which proves the result. The second part follows from arguments similar to the Proof of Lemma 2.
\end{proof}
}

Note that optimizing $F_a(X)$ is equivalent to optimizing $\sum_{i = 1}^l [f_i(X)]^a$ which is a higher order polynomial function. For example, when $a = 2$ we get the quadratic combinatorial program~\cite{buchheim2018quadratic} which can possibly result in a a $\sqrt{l}$ approximation given a exact or approximate quadratic optimizer. Unfortunately, optimizing this problem is NP hard for most constraints when $a \geq 2$. However, several heuristics and efficient solvers~\cite{benson1999mixed,lawler1963quadratic,buchheim2018quadratic,loiola2007survey} exist for solving this in the quadratic setting. Moreover, for special cases such as the assignment (matching) problem, specialized algorithms such as the Graduated Assignment algorithm~\cite{gold1996graduated} exist for solving this. While these are not guaranteed to theoretically achieve the optimal solution, they work quite in practice, thus yielding a family of heuristics for \textsc{Robust-Min}.

\subsection{Average Approximation based Algorithms} 
In this section, we shall look at a simple approximation of $\max_i f_i(X)$, which in turn shall lead to an approximation algorithm for \textsc{Robust-SubMin}. We first show that $f_{avg}(X) = 1/l \sum_{i = 1}^l f_i(X)$ is an $l$-approximation of $\max_{i = 1:l} f_i(X)$, which implies that minimizing $f_{avg}(X)$ implies an approximation for \textsc{Robust-SubMin}. The results in this Section are summarized in Table 1 (specifically, the \textsc{mmin-aa} and \textsc{ea-aa} columns of Table 1).
\begin{theorem}\label{robust-smin-aa}
Given a non-negative set function $f$, define $f_{avg}(X) = \frac{1}{l}  \sum_{i = 1}^l f_i(X)$. Then $f_{avg}(X) \leq \max_{i = 1:l} f_i(X) \leq lf_{avg}(X)$. Denote $\hat{X}$ as $\beta$-approximate optimizer of $f_{avg}$. Then $\max_{i = 1:l} f_i(\hat{X}) \leq l\beta \max_{i = 1:l} f_i(X^*)$ where $X^*$ is the exact minimizer of \textsc{Robust-SubMin}.
\end{theorem}
\arxiv{
 \begin{proof}
 To prove the first part, notice that $f_i(X) \leq \max_{i = 1:l} f_i(X)$, and hence $1/l \sum_i f_i(X) \leq \max_{i = 1:l} f_i(X)$. The other inequality also directly follows since the $f_i$'s are non-negative and hence $\max_{i = 1:l} f_i(X) \leq \sum_i f_i(X) = l f_{avg}(X)$. To prove the second part, observe that $\max_i f_i(\hat{X}) \leq lf_{avg}(\hat{X}) \leq l\beta f_{avg}(X^*) \leq l\beta \max_i f_i(X^*)$. The first inequality holds from the first part of this result, the second inequality holds since $\hat{X}$ is a $\beta$-approximate optimizer of $f_{avg}$ and the third part of the theorem holds again from the first part of this result.
 \end{proof}}
 
 Since $f_{avg}$ is a submodular function, we can use the majorization-minimization (which we call \textsc{mmin-aa}) and ellipsoidal approximation (\textsc{ea-aa}) for constrained submodular minimization~\cite{rkiyersemiframework2013,curvaturemin,goel2009optimal,jegelka2011-inference-gen-graph-cuts}. The following corollary provides the approximation guarantee of \textsc{mmin-aa} and \textsc{ea-aa} for \textsc{Robust-SubMin}.
 \begin{corollary}
 Using the majorization minimization (\textsc{mmin}) scheme with the average approximation achieves an approximation guarantee of $l K(|X^*|, \kappa_{avg})$ where $X^*$ is the optimal solution of \textsc{Robust-SubMin} and $\kappa_{avg}$ is the curvature of $f_{avg}$. Using the curvature-normalized ellipsoidal approximation algorithm from~\cite{curvaturemin,goemans2009approximating} achieves a guarantee of $O(lK(|\mathcal V|\log |\mathcal V|, \kappa_{avg}))$
 \end{corollary}
This corollary directly follows by combining the approximation guarantee of \textsc{mmin} and \textsc{ea}~\cite{curvaturemin,rkiyersemiframework2013} with Theorem~\ref{robust-smin-aa}. Substituting the values of $|\mathcal V|$ and $|X^*|$ for various constraints, we get the results in Table~1. While the average case approximation method provides a bounded approximation guarantee for \textsc{Robust-SubMin}, it defeats the purpose of the robust formulation. Moreover, the approximation factor is worse by a factor of $l$. Below, we shall study some techniques which directly try to optimize the robust formulation.
 
 \subsection{Majorization-Minimization Algorithm} The Majorization-Minimization algorithm is a sequential procedure which uses upper bounds of the submodular functions defined via supergradients. Starting with $X^0 = \emptyset$, the algorithm proceeds as follows. At iteration $t$, it constructs modular upper bounds for each function $f_i$, $m^{f_i}_{X^t}$ which is tight at $X^t$. We can use either one of the two modular upper bounds defined in Section 2. The set $X^{t+1} = \mbox{argmin}_{X \in \mathcal C} \max_i m^{f_i}_{X^t}(X)$. This is a min-max robust optimization problem. The following theorem provides the approximation guarantee for \textsc{mmin}.\looseness-1
 \begin{theorem} \label{thm3}
 If $l$ is a constant, \textsc{mmin} achieves an approximation guarantee of $(1 + \epsilon)K(|X^*|, \kappa_{wc})$ for the knapsack, spanning trees, matching and s-t path problems. The complexity of this algorithm is exponential in $l$. When $l$ is unbounded, \textsc{mmin} achieves an approximation guarantee of $lK(|X^*|, \kappa_{wc})$. For spanning trees and shortest path constraints, \textsc{mmin} achieves a \\ $O(\min(\log n, l) K(n, \kappa_{wc}))$ and a $O(\min(\sqrt{n}, l) K(n, \kappa_{wc}))$ approximation. Under cardinality and partition matroid constraints, \textsc{mmin} achieves a $O(\log l K(n, \kappa_{wc})/\log\log l)$ approximation.
 \end{theorem}
 
 Substituting the appropriate bounds on $|X^*|$ for the various constraints, we get the results in Tables 1 and 2. $\kappa_{wc}$ corresponds to the worst case curvature $\max_i \kappa_{f_i}$. 
 
 We now elaborate on the Majorization-Minimization algorithm. 
  At every round of the majorization-minimization algorithm we need to solve 
 \begin{align} \label{mmin-subproblem}
 X^{t+1} = \mbox{argmin}_{X \in \mathcal C} \max_i m^{f_i}_{X^t}(X).
 \end{align}
 We consider three cases. The first is when $l$ is a constant. In that case, we can use an FPTAS to solve Eq.~\eqref{mmin-subproblem}~\cite{aissi2009min,aissi2010general}. We can obtain a $1+\epsilon$ approximation with complexities of $O(n^{l+1}/\epsilon^{l-1})$ for shortest paths, $O(mn^{l+4}/\epsilon^l \log{n/\epsilon})$ for spanning trees, $O(mn^{l+4}/\epsilon^l \log{n/\epsilon})$ for matchings and  $O(n^{l+1}/\epsilon^l)$ for knapsack~\cite{aissi2010general}. The results for constant $l$ is shown in Table~2 (column corresponding to \textsc{mmin}). The second case is a generic algorithm when $l$ is not constant. Note that we cannot use the FPTAS since they are all exponential in $l$.  In this case, at every iteration of \textsc{mmin} can use the framework of approximations discussed in Section 3.2, and choose the solution with a better objective value. In particular, if we use the two modular bounds of the $\max$ function (i.e. the average the the max bounds), we can obtain $l$-approximations of $\max_i m^{f_i}_{X^t}$. One can also use the quadratic and cubic bounds which can possibly provide $\sqrt{l}$ or $l^{0.3}$ bounds. While these higher order bounds are still theoretically NP hard, there exist practically efficient solvers for various constraints (for e.g. quadratic assignment model for matchings~\cite{loiola2007survey}). Finally, for the special cases of spanning trees, shortest paths, cardinality and partition matroid constraints, there exist LP relaxation based algorithms which achieve approximation factors of $O(\log n)$, $O(\sqrt{n\log l/\log\log l}$, $O(\log l/\log \log l)$ and $O(\log l/\log \log l)$ respectively. The approximation guarantees of \textsc{mmin} for unbounded $l$ is shown in Table 1.

  \arxiv{
We now prove Theorem~\ref{thm3}.
 \begin{proof}
 Assume we have an $\alpha$ approximation algorithm for solving problem~\eqref{mmin-subproblem}.
 We start \textsc{mmin} with $X^0 = \emptyset$. We prove the bound for \textsc{mmin} for the first iteration. Observe that  $m^{f_i}_{\emptyset}(X)$ approximate the submodular functions $f_i(X)$ up to a factor of $K(|X|, \kappa_i)$~\cite{curvaturemin}. If $\kappa_{wc}$ is the maximum curvature among the functions $f_i$, this means that $m^{f_i}_{\emptyset}(X)$ approximate the submodular functions $f_i(X)$ up to a factor of $K(|X|, \kappa_{wc})$ as well. Hence $\max_i m^{f_i}_{\emptyset}(X)$ approximates $\max_i f_i(X)$ with a factor of $K(|X|, \kappa_{wc})$. In other words, 
 \begin{align}
 \max_i f_i(X) \leq \max_i m^{f_i}_{\emptyset}(X) \leq K(|X|, \kappa_{wc}) \max_i f_i(X)    
 \end{align}
 Let $\hat{X_1}$ be the solution obtained by optimizing  $m^{f_i}_{\emptyset}$ (using an  $\alpha$ approximation algorithms for the three cases described above). It holds that $\max_i m^{f_i}_{\emptyset}(\hat{X}_1) \leq \alpha \max_i m^{f_i}_{\emptyset}(X^m_1)$ where $X^m_1$ is the optimal solution of $\max_i m^{f_i}_{\emptyset}(X)$ over the constraint $\mathcal C$. Furthermore, denote $X^*$ as the optimal solution of $\max_i f_i(X)$ over $\mathcal C$. Then $\max_i m^{f_i}_{\emptyset}(X^m_1) \leq \max_i m^{f_i}_{\emptyset}(X^*) \leq K(|X^*|, \kappa_{wc})\max_i f_i(X^*)$. Combining both, we see that:
 \begin{align}
 \max_i f(\hat{X}_1) \leq \max_i m^{f_i}_{\emptyset}(\hat{X}_1) \leq \alpha K(|X^*|, \kappa_{wc})\max_i f_i(X^*)    
 \end{align}
 We then run \textsc{mmin} for more iterations and only continue if the objective value increases in the next round. Using the values of $\alpha$ for the different cases above, we get the results.
 \end{proof}
 }

\begin{table*}
\begin{center}
\scriptsize{
\begin{tabular}{ | c |  c |  }
\hline
 Constraints & $\hat{\mathcal P_{\mathcal C}}$  \\ \hline
 Matroids (includes Spanning Trees, Cardinality) & $\{x \in [0, 1]^n, x(S) \geq
r_{\mathcal M}(V) - r_{\mathcal M}(V \backslash S), \forall S
\subseteq V\}$ \\ \hline
 Set Covers (includes Vertex Covers and Edge Covers) & $\{x \in [0, 1]^{|\mathcal
  S|} \mid \sum_{i: u \in S_i} x(i) \geq c_u, \forall u \in \mathcal U\}$ \\ \hline
  s-t Paths & $\{x \in [0, 1]^{|\mathcal E|} \mid \sum_{e \in C}x(e) \geq 1$ , for every
s-t cut $C \subseteq \mathcal E\}$ \\ \hline
s-t Cuts & $\{x \in [0, 1]^{|\mathcal E|} \mid \sum_{e \in P}x(e) \geq 1$,
for every s-t path $P \subseteq \mathcal E\}$ \\ \hline
\end{tabular}}
\end{center}
\caption{The Extended Polytope $\hat{\mathcal P_{\mathcal C}}$ for many of the combinatorial constraints discussed in this paper. See~\cite{iyer2014monotone} for more details.}
\end{table*}

\subsection{Ellipsoidal Approximation Based Algorithm} 
Next, we use the Ellipsoidal Approximation to approximate the submodular function $f_i$. To account for the curvature of the individual functions $f_i$'s, we use the curve-normalized Ellipsoidal Approximation~\cite{curvaturemin}. We then obtain the functions $\hat{f}_i(X)$ which are of the form $(1 - \kappa_{f_i})\sqrt{w_{f_i}(X)} + \kappa_{f_i}\sum_{j \in X} f_i(j)$, and the problem is then to optimize $\max_i \hat{f}_i(X)$ subject to the constraints $\mathcal C$. This is no longer a min-max optimization problem. The following result shows that we can still achieve approximation guarantees in this case.
\begin{theorem}
For the case when $l$ is a constant, \textsc{ea} achieves an approximation guarantee of \\ $O(K(\sqrt{|\mathcal V|\log |\mathcal V|}, \kappa_{wc}))$ for the knapsack, spanning trees, matching and s-t path problems. The complexity of this algorithm is exponential in $l$. When $l$ is unbounded, the \textsc{ea} algorithm achieves an approximation guarantee of $O(\sqrt{l}\sqrt{|\mathcal V|} \log |\mathcal V|)$ for all constraints. In the case of spanning trees, shortest paths, the \textsc{ea} achieves approximation factors of $O(\min(\sqrt{\log n}, \sqrt{l}) \sqrt{m}\log m)$, and $O(\min(n^{0.25}, \sqrt{l})\sqrt{m}\log m)$. Under cardinality and partition matroid constraints, \textsc{ea} achieves a $O(\sqrt{\log l/\log\log l} \sqrt{n}\log n)$ approximation.
\end{theorem}
For the case when $l$ is bounded, we reduce the optimization problem after the Ellipsoidal Approximation into a multi-objective optimization problem, which provides an FPTAS for knapsack, spanning trees, matching and s-t path problems~\cite{papadimitriou2000approximability,mittal2013general}. When $l$ is unbounded, we further reduce the \textsc{ea} approximation objective into a linear objective which then provides the approximation guarantees similar to \textsc{mmin} above. However, as a result, we loose the curvature based guarantee. \narxiv{A detailed proof is in the extended version of this paper. }
\arxiv{
\begin{proof}
First we start with the case when $l$ is a constant. Observe that the optimization problem is 
\begin{align*}
\min_{X \in \mathcal C} \max_i \hat{f}_i(X) = \min_{X \in \mathcal C} \max_i (1 - \kappa_{f_i})\sqrt{w_{f_i}(X)} + \\ \kappa_{f_i}\sum_{j \in X} f_i(j)
\end{align*}
This is of the form $\min_{X \in \mathcal C} \max_i \sqrt{w^i_1(X)} + w^i_2(X)$. Define: 
\begin{align}
h(y^1_1, y^1_2, y^2_2, y^2_2, \cdots, y^l_1, y^l_2) = \max_i \sqrt{y^i_1} + y^i_2    
\end{align}
Note that the optimization problem is $\min_{X \in \mathcal C} h(w^1_1(X), w^1_2(X), \cdots, w^l_1(X), w^l_2(X))$. Observe that $h(\mathbf{y}) \leq h(\mathbf{y^{\prime}})$ if $\mathbf{y} \leq \mathbf{y^{\prime}}$. Furthermore, note that $\mathbf{y} \geq 0$. Then given a $\lambda > 1$, $h(\lambda \mathbf{y}) = \max_i \sqrt{\lambda y^i_1} + \lambda y^i_2 \leq \lambda \sqrt{y^i_1} + \lambda y^i_2 \leq \lambda h(\mathbf{y})$. As a result, we can use Theorem 3.3 from~\cite{mittal2013general} which provides an FPTAS as long as the following exact problem can be solved on $\mathcal C$: Given a constant $C$ and a vector $c \in \mathbf{R}^n$, does there exist a $x$ such that $\langle c, x \rangle = C$? A number of constraints including matchings, knapsacks, s-t paths and spanning trees satisfy this~\cite{papadimitriou2000approximability}. For these constraints, we can obtain a $1 + \epsilon$ approximation algorithm in complexity exponential in $l$. 

When $l$ is unbounded, we directly use the Ellipsoidal Approximation and the problem then is to optimize $\min_{X \in \mathcal C} \max_i \sqrt{w_{f_i}(X)}$. We then transform this to the following optimization problem: \\ $\min_{X \in \mathcal C} \max_i w_{f_i}(X)$. Assume we can obtain an $\alpha$ approximation to the problem \\ $\min_{X \in \mathcal C} \max_i w_{f_i}(X)$. This means we can achieve a solution $\hat{X}$ such that:
\begin{align}
 \max_i w_{f_i}(\hat{X}) \leq \alpha \max_i w_{f_i}(X^{ea})   
\end{align}
where $X^{ea}$ is the optimal solution for the problem $\min_{X \in \mathcal C} \max_i w_{f_i}(X)$. Then observe that $\max_i \sqrt{w_{f_i}(X^{ea})} \leq \max_i \sqrt{w_{f_i}(X^*)} \leq \max_i f_i(X^*)$. Combining all the inequalities and also using the bound of the Ellipsoidal Approximation, we have $\max_i f_i(\hat{X}) \leq \beta \max_i \sqrt{w_{f_i}(\hat{X})} \leq \beta \sqrt{\alpha} \max_i \sqrt{w_{f_i}(X^{ea})} \leq \beta \sqrt{\alpha} \max_i \sqrt{w_{f_i}(X^*)} \leq \beta \sqrt{\alpha} \max_i f_i(X^*)$ where $\beta$ is the approximation of the Ellipsoidal Approximation.

We now use this result to prove the theorem. Consider two cases. First, we optimize the \emph{avg} and \emph{max} versions of $\max_i w_{f_i}(X)$ which provide $\alpha = l$ approximation. Secondly, for the special cases of spanning trees, shortest paths, cardinality and partition matroid constraints, there exist LP relaxation based algorithms which achieve approximation factors $\alpha$ being $O(\log n)$, $O(\sqrt{n\log l/\log\log l})$, $O(\log l/\log \log l)$ and $O(\log l/\log \log l)$ respectively. Substitute these values of $\alpha$ and using the fact that $\beta = O(\sqrt{|V|}\log |V|)$, we get the approximation bound.
\end{proof}
}

\subsection{Continuous Relaxation Algorithm}
Here, we use the continuous relaxation of a submodular function. In particular, we use the relaxation $\max_i \hat{f_i}(x), x \in [0, 1]^{|\mathcal V|}$ as the continuous relaxation of the original function $\max_i f_i(X)$ (here $\hat{f}$ is the \lovasz{} extension). Its easy to see that this is a continuous relaxation. Since the \lovasz{} extension is convex, the function $\max_i \hat{f_i}(x)$ is also a convex function. This means that we can exactly optimize the continuous relaxation over a convex polytope. The remaining question is about the rounding and the resulting approximation guarantee due to the rounding. Given a constraint $\mathcal C$, define the up-monotone polytope similar to~\cite{iyer2014monotone} $\hat{\mathcal P_{\mathcal C}} = \mathcal P_{\mathcal C} + [0,1]^{|\mathcal V|}$. We then use the observation from~\cite{iyer2014monotone} that all the constraints considered in Table 1 can be expressed as:
\begin{align}
\hat{\mathcal P_{\mathcal C}} = \{x \in [0,1]^n | \sum_{i \in W} x_i \geq b_W \mbox{, for all $W \in \mathcal W$}\}. 
\end{align}

We then round the solution using the following rounding scheme. Given a continuous vector $\hat{x}$ (which is the optimizer of $\max_i \hat{f_i}(x), x \in \hat{\mathcal P_{\mathcal C}}$, order the elements based on $\sigma_{\hat{x}}$. Denote $X_i = [\sigma_{\hat{x}}[1], \cdots, \sigma_{\hat{x}}[i]]$ so we obtain a chain of sets $\emptyset \subseteq X_1 \subseteq X_2 \subseteq \cdots \subset X_n$. Our rounding scheme picks the smallest $k$ such that $X_k \in \hat{\mathcal C}$. Another way of checking this is if there exists a set $X \subset X_k$ such that $X \in \mathcal C$. Since $\hat{\mathcal C}$ is up-monotone, such a set must exist.  The following result shows the approximation guarantee.
\begin{theorem}
Given submodular functions $f_i$ and constraints $\mathcal C$ which can be expressed as $\{x \in [0,1]^n | \sum_{i \in W} x_i \geq b_W$ for all $W \in \mathcal W$\} for a family of sets $\mathcal W = \{W_1, \cdots\}$, the continuous relaxation scheme achieves an approximation guarantee of $\max_{W \in \mathcal W} |W| - b_W + 1$. If we assume the sets in $\mathcal W$ are disjoint, the integrality bounds matches the approximation bounds. 
\end{theorem}
\begin{proof}
The proof of this theorem is closely in line with Lemma 2 and Theorem 1 from~\cite{iyer2014monotone}. We first show the first part. Given monotone submodular functions $f_i, i \in 1, \cdots, l$, and an optimizer $\hat{x}$ of $\max_i \hat{f}(x)$, define $\hat{X_{\theta}} = \{i: \hat{x}_i \geq \theta\}$. We choose $\theta$ such that $\hat{X_{\theta}} \in \mathcal C$. Then  $\max_i f_i(\hat{X_{\theta}}) \leq 1/\theta \max_i f_i(X^*)$ where $X^*$ is the optimizer of $\min_{X \in \mathcal C} \max_i f_i(X)$. To prove this, observe that, by definition $\theta 1_{\hat{X_{\theta}}} \leq \hat{x}$\footnote{$1_A$ is the indicator vector of set $A$ such that $1_A[i] = 1$ if and only if $i\in A$.}. As a result, $\forall i, \hat{f_i}(\theta 1_{\hat{X_{\theta}}}) = \theta f_i(\hat{X_{\theta}}) \leq \hat{f_i}(\hat{x})$ -- this follows because of the positive homogeneity of the \lovasz{} extension. This implies that $\theta \max_i f_i(\hat{X_{\theta}}) \leq \hat{f_i}(\hat{x}) \leq \min_{x \in \mathcal P_{\mathcal C}} \hat{f_i}(\hat{x}) \leq \min_{X \in \mathcal C} \max_i f_i(X)$. The last inequality holds from the fact that the discrete solution is greater than the continuous one since the continuous one is a relaxation. This proves this part of the theorem.

Next, we show that the approximation guarantee holds for the class of constraints defined as $\{x \in [0,1]^n | \sum_{i \in W} x_i \geq b_W\}$. From these constraints, note that for every $W \in \mathcal W$, at least $b_W \leq |W|$ elements need to ``covered''.  Consequently, to round a vector $x \in \hat{\mathcal P_{\mathcal C}}$, it is sufficient to choose $\theta = \min_{W \in \mathcal W} x_{[b_W, W]}$ as the rounding threshold, where $x_{[k, A]}$ denotes the $k^{\mbox{th}}$ largest entry of $x$ in a set $A$. 
The worst case scenario is that the $b_W - 1$ entries of $x$ with indices in the set $W$ are all $1$, and the remaining mass of $1$ is equally distributed over the remaining elements in $W$. In this case, the value of $x_{[b_W, W]}$ is $1/(|W| - b_W + 1)$. 
Since the constraint requires $\sum_{i \in W} x_i \geq b_W$, it must hold that $x_{[b_W, W]} \geq 1/(|W| - b_W + 1)$. Combining this with the first part of the result proves this theorem.
\end{proof}

We can then obtain the approximation guarantees for different constraints including cardinality, spanning trees, matroids, set covers, edge covers and vertex covers, matchings, cuts and paths by appropriately defining the polytopes $\mathcal P_{\mathcal C}$ and appropriately setting the values of $\mathcal W$ and $\max_{W \in \mathcal W} |W| - b_W + 1$. Combining this with the appropriate definitions (shown in Table 3), we get the approximation bounds in Table 1.

\subsection{Analysis of the Bounds for Various Constraints} 
Given the bounds in Tables 1 and 2, we discuss the tightness of these bounds viz-a-via the hardness. In the case when $l$ is a constant, \textsc{mmin} achieves tight bounds for Trees, Matchings and Paths while the \textsc{ea} achieves tight bounds up to $\log$ factors for knapsack constraints. In the case when $l$ is not a constant, \textsc{mmin} achieves a tight bound up to $\log$ factors for spanning tree constraints. The continuous relaxation scheme obtains tight bounds in the case of vertex covers. In the case when the functions $f_i$ have curvature $\kappa = 1$, CR also obtains tight bounds for edge-covers and matchings. We also point out that the bounds of average approximation (AA) depend on the average case curvature as opposed to the worst case curvature. However, in practice, the functions $f_i$ often belong to the same class of functions in which case all the functions $f_i$ have the same curvature.

\section{\textsc{Robust-SubMax}}
In this section, we will study \textsc{Robust-SubMax}. Recall that the \textsc{Robust-SubMax} optimization problem is:
\begin{align*}
    \mbox{Problem 2: } \max_{X \in \mathcal C} \min_{i = 1:k} g_i(X)
\end{align*}

\cite{krause08robust} show that \textsc{Robust-SubMax} is inapproximable (in a single criteria manner) upto any polynomial factor unless P = NP even with cardinality constraints. This is in contrast to \textsc{Robust-SubMin}, which achieves bounded approximation factors (Section 3). \cite{krause08robust} provide a bi-criteria approximation factor for \textsc{Robust-SubMax} under a single cardinality constraint. \cite{anari2019robust} extend this to matroid and knapsack constraints. \cite{krause08robust,anari2019robust} provide a bi-criteria approximation factor of $(1 - \epsilon, O(\log k/\epsilon))$ for cardinality constraints.  For Matroid constraints, they provide bi-criteria algorithms by creating a union of $O(\log k/\epsilon)$ independent sets, with the union set having a guarantee of $1 - \epsilon$. This result can be extended to multiple matroid constraints~\cite{anari2019robust}. Anari et al~\cite{anari2019robust} also obtained a $(1 - \epsilon, O(\log k/\epsilon))$ approximation for a single knapsack constraint. In the following theorem, we complete the picture of \textsc{Robust-SubMax} by providing a bi-criteria approximation for \textsc{Robust-SubMax} under multiple knapsack constraints. In particular, we consider the following problem:
\begin{align*}
    \max_{X \subseteq V} \min_{i = 1:k} g_i(X), \mbox{ s.t } w_i(X) \leq b_i
\end{align*}
Here $w_i$ are modular functions.
\begin{theorem}
\label{robust-submax-mulknapsack}
Using a modified greedy algorithm, we achieve a $(1 - \epsilon, O(l.\mbox{ln }\frac{k}{\epsilon}))$ bi-criteria approximation for \textsc{Robust-SubMax} under $l$ knapsack constraints. Using the continuous greedy algorithm, we can achieve an improved factor of $(1 - \epsilon, O(\mbox{ln }\frac{k}{\epsilon}))$ for the same problem.
\end{theorem} 
Since the greedy algorithm does not work directly for multiple knapsack constraints, we convert the problem first into a single knapsack constraint which provides the bi-criteria approximation. The approximation is worse by a factor $l$. Next, we use the continuous greedy algorithm, in a manner similar to~\cite{anari2019robust} and with the rounding scheme of~\cite{kulik2009maximizing}, to achieve the tight approximation factor of $(1 - \epsilon, O(\mbox{ln }\frac{k}{\epsilon}))$. Note that this bound is independent of the number of knapsack constraints $l$.
\arxiv{
\begin{proof}
We start with the optimization problem: 
\begin{align}
\max_{X \subseteq V} \min_{i = 1:k} g_i(X) \,\,\ | \,\,\ w_i(X) \leq b_i, i = 1, \cdots, l.
\end{align}
Note that the constraints $w_i(X) \leq b_i$ can equivalently be written as $\max_i \frac{w_i(X)}{b_i} \leq 1$. We can then define two approximations of this. One is the \emph{modmax} approximation, which is \\ $\sum_{i \in X} \max_{j \in 1:l} w_{ij}/b_i$ and the other is the \emph{avg} approximation: $\sum_{i \in X} \sum_{j \in 1:l} w_{ij}/lb_i$. Both these are $l$-approximations of $\max_i \frac{w_i(X)}{b_i} \leq 1$. Denote these approximations as $\hat{w}(X)$ and w.l.o.g assume that $\hat{w}(X) \leq \max_i w_i(X)/b_i \leq l\hat{w}(X)$. We can then convert this to an instance of \textsc{Robust-SubMax} with a single knapsack constraint:
\begin{align}
\max_{X \subseteq V} \min_{i = 1:k} g_i(X) \,\,\ | \,\,\ \hat{w}(X) \leq 1
\end{align}
From~\cite{anari2019robust}, we know that we can achieve a $(1 - \epsilon, \log k/\epsilon)$ bi-criteria approximation. In other words, we can achieve a solution $\hat{X}$ such that $\min_i g_i(\hat{X}) \geq (1 - \epsilon) \min_i g_i(X^*)$ and $\hat{w}(\hat{X}) \leq \log k/\epsilon$. Since we have that $\max_i w_i(X)/b_i \leq l\hat{w}(X)$, this implies $\max_i w_i(X)/b_i \leq l \mbox{ln } k/\epsilon$. This completes the first part. In practice, we can use both the approximations (i.e. the \emph{modmax} and the \emph{avg} approximation and choose the better among the solutions).

Next, we prove the second part of the theorem. This uses the continuous greedy algorithm and we use the proof technique from~\cite{anari2019robust}. In particular, first truncate $g_i$'s to $c$, so we can define $g_i^c(X) = \min(g_i(X), c)$. Define the multi-linear extension of $g_i^c$ as $G_i^c$. First we argue that we can obtain a solution $y(\tau)$ at time $\tau$ such that $G^c_i(y(\tau)) \geq (1 - e^{-\tau}) c, \forall i$ where $y(\tau) \in \tau \mathcal P_{\mathcal C}$. In other words, $y(\tau)$ satisfies $\langle y, w_i \rangle/b_i \leq \tau, \forall i$. This follows from Claim 1 in~\cite{anari2019robust} since the result holds for any down-monotone polytope~\cite{chekuri2011submodular}. Next, we use the rounding technique from~\cite{kulik2009maximizing}. In order to do this, we first set $\tau = \mbox{ln } k/\epsilon$ so we can achieve a solution so $G_i^c(y) \geq (1 - \epsilon/k)c$. Notice that this implies $G^c(y) = \sum_i G_i^c(y) \geq (1 - \epsilon/k)c$. Since $G^c$ is a single submodular function, we can round the obtained fractional solution $y$ using the rounding scheme from~\cite{kulik2009maximizing}. This will achieve a discrete solution $X$ such that $g^c(X) \geq (1 - \epsilon/k - \epsilon^{\prime})c$ and satisfies $\max_i w_i(X)/b_i \leq \mbox{ln }\frac{k}{\epsilon}$. The question is, what can we say about the original objective. We can show that for all $i = 1:k$, $g^c_i(X) \geq (1 - \epsilon - k\epsilon^{\prime})c$ since suppose this were not the case then there would exist at least one $i$ such that $g^c_i(X) < (1 - \epsilon - k\epsilon^{\prime})c$. Since $g^c_j \leq c, \forall j = 1:k$, this implies that $\sum_{j=1}^k g^c_j/l < c(k-1)/k + (1 - \epsilon - k\epsilon^{\prime})c < c(1 - \epsilon/k - \epsilon^{\prime})$ which refutes the fact that $g^c(X) = \sum_{j=1}^l g^c_j/k \geq (1 - \epsilon/k - \epsilon^{\prime})c$. Define a new $epsilon$ as $\epsilon + l\epsilon^{\prime}$ and this shows that we can achieve a discrete solution $X$ such that $g^c_i(X) < (1 - \epsilon)c$ and $\max_i w_i(X)/b_i \leq \mbox{ln }\frac{k}{\epsilon}$. We can then the binary search over the values of $c$ (similar to~\cite{anari2019robust}) and obtain a $(1 - \epsilon, O(\mbox{ln }\frac{k}{\epsilon}))$ bicriteria approximation.
\end{proof}
}

\section{\textsc{Robust-scsc} and \textsc{Robust-scsk}}
We begin by restating the optimization problems, \textsc{Robust-scsc} and \textsc{Robust-scsk}:
\begin{align*} 
    \mbox{Problem 3: } \min_{X \subseteq V}  \max_{i=1:l} f_i(X) 
    \,\, | \,\, g_i(X) \geq c_i, i = 1, \cdots k \nonumber \\
    \mbox{Problem 4: }\max_{X \subseteq V} \min_{i = 1:k} g_i(X)
    \,\, | \,\, f_i(X) \leq b_i, i = 1, \cdots l \nonumber
\end{align*}
Next, we define some notations and constructs (in addition to those defined in Section 2). Define $\delta^k_g = \log \max_{v \in V} \sum_{i = 1}^k \min(g_i(v), 1)$. Also when there is a single function $g$, define $\delta_g = \log \max_{v \in V} g(v)$. Finally, we introduce the notion of a bi-criteria approximation. A $(\alpha, \beta)$-bicriteria approximation for \textsc{Robust-scsc} achieves a set $\hat{X}$ such that $\max_i f_i(\hat{X}) \leq \alpha \max_i f_i(X^*)$ and $g_i(\hat{X}) \geq \beta c_i, \forall i$. Similarly a $(\beta, \alpha)$-bicriteria approximation for \textsc{Robust-scsk} achieves a set $\hat{X}$ such that $\min_i g_i(\hat{X}) \leq \beta \min_i g_i(X^*)$ and $\max_i f_i(\hat{X}) \leq \alpha b_i, \forall i$. $\alpha, \beta$ satisfy $\alpha \geq 1$ and $\beta \leq 1$. Finally, we shall often use $F(X)$ in place of $\max_{i=1:l} f_i(X)$ and $G(X)$ in the place of $\min_{i = 1:k} g_i(X)$. In the sections below, we shall study several important results including the connections between the \textsc{Robust-scsc} and \textsc{Robust-scsk}, hardness bounds and approximation algorithms. 

\subsection{Relation between \textsc{Robust-scsc} and \textsc{Robust-scsk}}
We first start by showing that \textsc{Robust-scsc} and \textsc{Robust-scsk} are closely related and an approximation algorithm for one of the problems provides an approximation algorithm for the other. Without loss of generality, we can assume $b_i = 1, c_i = 1$ and obtain the following problems:
\begin{align}
    \label{eqref1} \mbox{Problem 3}: \min_{X \subseteq V}  \max_{i=1:l} f_i(X)
    \,\, | \,\, \min_{i=1:k} g_i(X) \geq 1,  \\
  \label{eqref2} \mbox{Problem 4}: \max_{X \subseteq V} \min_{i = 1:k} g_i(X)
    \,\, | \,\, \max_{i=1:l} f_i(X) \leq 1
\end{align}
Henceforth we shall use equations~\eqref{eqref1} and \eqref{eqref2} for the formulations of \textsc{Robust-scsc} and \textsc{Robust-scsk}. Furthermore define \textsc{Robust-scsc}($c$) and \textsc{Robust-scsk}($b$) as generalizations where the constraints are $\min_{i=1:k} g_i(X) \geq c$ and $\max_{i=1:l} f_i(X) \leq b$ respectively. Algorithm 1 and 2 provide a binary search procedure to obtain bi-criteria approximation algorithms for one of the problems given a bi-criteria approximation algorithm for the other. However, Algorithms 1 and 2 should be run only if a) there exists a solution to both \textsc{Robust-scsc} and \textsc{Robust-scsk}, and b) The solution is not trivially the ground set $V$, and (a) can be checked by just ensuring the constraints can be satisfied.
    \begin{algorithm}
    \caption{Binary Search: \textsc{Robust-scsc} to \textsc{Robust-scsk}}
    \begin{algorithmic}[1]
      \STATE \textbf{Input:} An instance of \textsc{Robust-scsk}, and
      a $[\alpha, \beta]$ approximation algorithm for \textsc{Robust-scsc}, and $\epsilon \in [0,1)$.
    \STATE \textbf{Output: } $[(1 - \epsilon) \beta, \alpha]$ approx. for \textsc{Robust-scsk}.
    \STATE $c_{\min} \leftarrow \min_j \min_i g_i(j), c_{\max} \leftarrow \min_i g_i(V)$
\WHILE{$c_{\max} - c_{\min} \geq \epsilon c_{\max}$ }
\STATE $c \leftarrow [c_{\max} + c_{\min}]/2$
\STATE $\hat{X_c} \leftarrow [\alpha, \beta]$ approx. for \textsc{Robust-scsc}($c$).
\IF{$\max_i f_i (\hat{X_c}) > \alpha$}
\STATE $c_{\max} \leftarrow c$
\ELSE
\STATE $c_{\min} \leftarrow c$
\ENDIF
    \ENDWHILE
\STATE Return $\hat{X}_{c_{\min}}$
    \end{algorithmic}
    \label{alg:alg3}
    \end{algorithm}

    \begin{algorithm}[]
     \caption{Binary Search: \textsc{Robust-scsk} to \textsc{Robust-scsc}.}
    \begin{algorithmic}[1]
      \STATE \textbf{Input:} An instance of \textsc{Robust-scsc}, a
      $[\beta, \alpha]$ approximation algorithm for \textsc{Robust-scsk}, and $\epsilon > 0$.
    \STATE \textbf{Output: } $[ (1 + \epsilon)\alpha, \beta]$ approx. for \textsc{Robust-scsc}.
\STATE $b_{\min} \leftarrow \min_j \max_i f_i(j), b_{\max} \leftarrow \max_i f_i(V)$.
\WHILE{$b_{\max} - b_{\min} \geq \epsilon b_{\min}$ }
\STATE $b \leftarrow [b_{\max} + b_{\min}]/2$
\STATE $\hat{X_b} \leftarrow [\beta, \alpha]$ approx. for Problem2($b$)
\IF{$g(\hat{X_b}) < \beta$}
\STATE $b_{\min} \leftarrow b$
\ELSE
\STATE $b_{\max} \leftarrow b$
\ENDIF
    \ENDWHILE
    \STATE Return $\hat{X}_{b_{\max}}$
    \end{algorithmic}
    \label{alg:alg4}
    \end{algorithm}

\begin{theorem}\label{robust-scsc-scsk-eq}
Algorithm 1 achieves a $((1 - \epsilon)\rho, \sigma)$ approximation for \textsc{Robust-scsk} using \\ $\log\frac{\min_i g_i(V)/min_j \min_i g_i(j)}{\epsilon}$ calls to a $(\sigma, \rho)$ bicriteria approximation algorithm for \textsc{Robust-SCSC}. Similarly, Algorithm 2 achieves a $((1 + \epsilon)\sigma, \rho)$ bicriteria approximation for \textsc{Robust-SCSC} with $\log\frac{\max_i f_i(V)/min_j \max_i f_i(j)}{\epsilon}$ calls to a $(\rho, \sigma)$ bicriteria approximation for \textsc{Robust-scsk}.
\end{theorem}
\begin{proof}
Note that both $F$ and $G$ are monotone functions. Let $c = c_{\min}$ and $c^{\prime} = c_{\max}$. An important observation is that throughout the algorithm, the values of $c_{\min}$ and $c_{\max}$ satisfy $F(\hat{X}_{c_{\min}}) \leq \alpha$ and $F(\hat{X}_{c_{\max}}) > \alpha $.
This holds because of the following arguments. We start with $F(\hat{X}_{c_{min}})\leq \alpha$. Note that we begin the algorithm by setting
$c_{min}$ as $min_j G_i(j)$ and $c_{max}$ as $G(V)$. Given that $\hat{X}_{c_{min}}$ is an $(\alpha,\beta)$ approximation of \textsc{Robust-scsk}($c_{min}$),
it holds that $F(\hat{X}_{c_{min}}) \leq \alpha F(X^*_{c_{min}})$ where $X^*_{c_{min}}$ is the optimal solution of \textsc{Robust-scsk}($c_{min}$).
Note that $c_{min}$ is the smallest possible value of the function G. As a result, the optimal solution will be the minimum possible value of F,
i.e. $argmin_j F(j)$. Next, we show that $F(X^*_{c_{min}}) \leq 1$. This holds because, as discussed above, there exists a solution for \textsc{Robust-scsk} and there exists a solution $X^*$ such that
$F(X^*) \leq 1$. Note that since $X^*_{c_{min}}$ is the minimum possible value of $F$, $F(X^*_{c_{min}}) \leq F(X^*) \leq 1$. Now that we have shown that $F(\hat{X}_{c_{min}}) \leq \alpha$ at the beginning, Algorithm 1 always maintains this inequality. A symmetric argument holds for the other case and we skip it in the interest of space. 

Since the values of $c_{\min}$ and $c_{\max}$ satisfy $F(\hat{X}_{c_{\min}}) \leq \alpha$ and $F(\hat{X}_{c_{\max}}) > \alpha $., it holds that $F(\hat{X_{c}}) \leq \alpha$ and $F(\hat{X_{c^{\prime}}}) > \alpha $. Moreover, notice that at convergence $c^{\prime}/c = c_{\max}/c_{\min} = 1/(1 - \epsilon)$. Denoting
  $X^*_{c^{\prime}}$ as the optimal solution for \textsc{Robust-scsk}($c^{\prime}$),
  we have that $F(X^*_{c^{\prime}}) > 1$ (a fact which follows from
  the observation that $\hat{X_c}$ is a $[\alpha, \beta]$ approximation
  of \textsc{Robust-scsk}($c$). Hence if $X^*$ is the optimal solution of \textsc{Robust-scsk}, it
  follows that $G(X^*) < c^{\prime}$. The reason for this is that,
  suppose, $G(X^*) \geq c^{\prime}$. Then it follows that $X^*$ is a
  feasible solution for \textsc{Robust-scsk} ($c^{\prime}$) and hence $F(X^*) \geq
  F(X^*_{c^{\prime}}) > b$. This contradicts the fact that $X^*$ is an
  optimal solution for \textsc{Robust-scsk} (since it is then not even feasible). Next, notice that $\hat{X_c}$ satisfies that $G(\hat{X_c}) \geq \beta
  c$, using the fact that $\hat{X_c}$ is obtained from a $(\alpha,
  \beta)$ bi-criterion algorithm for \textsc{Robust-scsk}($c$). Hence,
\begin{align}
G(\hat{X_c}) 
\geq \beta c 
= \beta (1 - \epsilon) c^{\prime} 
> \beta (1 - \epsilon) G(X^*) 
\end{align}
We already saw that $F(\hat{X_{c}}) \leq \alpha$ and hence Algorithm 1 is a $((1 - \epsilon) \beta,
\alpha)$ approximation for \textsc{Robust-scsk}. In order to show the complexity, notice that $c_{\max} - c_{\min}$ is decreasing throughout the algorithm. At the beginning, $c_{\max} - c_{\min} \leq G(V)$ and at convergence, $c_{\max} - c_{\min} \geq \epsilon c_{\max}/2 \geq \epsilon \min_j G(j)/2$. The bound at convergence holds since, let $c^{\prime}_{max}$ and $c^{\prime}_{min}$ be the values at the previous step. It holds that $c^{\prime}_{\max} - c^{\prime}_{\min} \geq \epsilon c^{\prime}_{\max}$. Moreover, $c_{max} - c_{min} = (c^{\prime}_{\max} - c^{\prime}_{\min})/2 \geq \epsilon c^{\prime}_{\max}/2 \geq \epsilon c_{\max}$. Hence the number of iterations is bounded by $\log_2 \frac{[2G(V)/\min_j G(j)]}{\epsilon}$. 

The second part (analysis of Algorithm 2) can be shown similarly using a
symmetric argument.
\end{proof}

\subsection{Hardness of \textsc{Robust-scsc} and \textsc{Robust-scsk}}
Next we discuss the hardness of \textsc{Robust-scsc} and \textsc{Robust-scsk}. We shall assume the functions $f_i$'s and $g_i$'s belong to the same family and hence have the same curvature. The result below holds for the general case, i.e. $k > 1, l > 1$. Better approximations can be obtained with $k = 1$ or $l = 1$.
\begin{theorem}
No approximation algorithm can achieve a $(\alpha, \beta)$-bicriteria approximation factor satisfying $\alpha/\beta < \Omega(\max\{K(\sqrt{n}, \kappa_f), 1 + \delta^l_g\})$ when $f_i$'s and $g_i$'s are submodular. When $f_i$'s are modular and $g_i$'s are submodular, no approximation algorithm can achieve a bi-criteria factor satisfying $\alpha/\beta < 1 + \delta^l_g$. Moreover, \textsc{Robust-scsk} does not admit any single criteria approximation factor when $k > 1$ unless P = NP.\looseness-1
\end{theorem}
\begin{proof}
Note that \textsc{Robust-scsc} generalizes \textsc{SCSC} which yields a lower bound $\Omega(K(\sqrt{n}, \kappa))$ \cite{nipssubcons2013}. \textsc{Robust-scsk} also  generalizes robust submodular maximization under cardinality and knapsack constraints and which also gives a lower bound of $1 + \delta_g$~\cite{krause08robust}. Moreover, this generalization also proves that one cannot achieve any polynomial single-factor approximation for \textsc{Robust-scsk} when $k > 1$ unless P = NP~\cite{krause08robust}.
\end{proof}

In the following sections, we shall study approximation algorithms under various scenarios of $f$'s and $g$'s depending on whether they are modular or submodular. In particular, we shall see that the submodularity of $g_i$'s doesn't affect the approximation as much. So we divide the results into two sections depending whether the $f_i$'s are modular or submodular.\looseness-1

\subsection{Modular $f_i$'s and Submodular (or Modular) $g_i$'s} \label{mod_f_submod_g_sec}
We show the approximation guarantee under three cases: a) $k = 1, l \geq 1$, b) $k \geq 1, l = 1$, and c) $k \geq 1, l \geq 1$. 
\begin{theorem} \label{mod_f_submod_g_sec-theorem-complete}
The following approximation guarantees hold when the functions $f_i$'s are Modular and $g_i$'s are Modular or Submodular:
\begin{enumerate}
    \item Case I, $k = 1, l \geq 1$: When the function $g$ is Submodular, 	\textsc{Robust-scsc}  admits a $l(1 + \delta_g)$ single-criteria approximation. There exists a Continuous Greedy based algorithm which admits a $(1 - \epsilon, 1 - 1/e - \epsilon)$ bi-criteria approximation for 	\textsc{Robust-scsc}. and a $1 - 1/e - \epsilon$ single factor approximation algorithm for 	\textsc{Robust-scsk}. 
    There exists a faster and more practical Discrete Greedy based algorithm which provides a worse bi-criteria factor of $(l(1 - \epsilon), 1/2(1 - 1/e - \epsilon))$ for 	\textsc{Robust-scsc}  and $(1/2(1 - 1/e), l)$ bicriteria factor for 	\textsc{Robust-scsk}. When $g$ is Modular, 	\textsc{Robust-scsc}  admits a $2l$ single factor approximation and a bi-criteria FPTAS, and 	\textsc{Robust-scsk}  admits a single factor FPTAS.
    \item Case II, $k \geq 1, l = 1$: When $g$ is either submodular or modular, 	\textsc{Robust-scsc}  admits a $1 + \delta_g^k$ single factor approximation algorithm, while 	\textsc{Robust-scsk}  admits $(1 - \epsilon, \max(1 + \delta_g^k, \log k/\epsilon))$ bi-criteria approximation algorithm.
    \item Case III, $k \geq 1, l \geq 1$: 	\textsc{Robust-scsc}  admits a $l(1 + \delta_g^k)$ single factor approximation and a $(O(\log k/\epsilon), 1 - \epsilon)$ bi-criteria approximation based on a \emph{Continuous greedy} algorithm. For 	\textsc{Robust-scsk} , the Continuous Greedy provides a $(1 - \epsilon, O(\log k/\epsilon))$ bicriteria approximation. There exists more efficient discrete greedy algorithm which provides an approximation guarantee of $(O(\log k/\epsilon), l(1 - \epsilon))$ and $(l(1 - \epsilon), O(\log k/\epsilon))$ for \textsc{Robust-scsc}  and \textsc{Robust-scsk}  respectively,
\end{enumerate}
\end{theorem}
Note the following. First, when $k = 1$, 	\textsc{Robust-scsc}  becomes Robust Submodular Set Cover (\textsc{Robust-ssc}), a problem which has not yet been formally studied. Two special cases of Theorem~\ref{mod_f_submod_g_sec-theorem-complete} have been studied. When $k = 1$, \textsc{Robust-scsk}  reduces to submodular maximization subject to multiple knapsack constraints~\cite{kulik2009maximizing}. Similarly, when $k \geq 1, l = 1$, 	\textsc{Robust-scsk}  generalizes \textsc{Robust-SubMax} under a single knapsack constraint which was studied in~\cite{anari2019robust}. 

Before proving this result, we show an important result we shall use. Define $F(X) = \max_i f_i(X)$ when $f_i$'s are modular. Consider two relaxations of $F$. The first is $\hat{F}(X) = \sum_{i \in X} \max_{j = 1:l} f_j(i)$ and the second is $\tilde{F}(X) = 1/l \sum_{i = 1}^l \sum_{j \in X} f_i(j) = \sum_{i = 1}^l f_i(X)$.
\begin{lemma}\label{mod-approximations}
 Given $f_i(j) \geq 0$, it holds that $\hat{F}(X) \geq F(X) \geq \frac{1}{l} \hat{F}(X)$. Furthermore, $\tilde{F}(X) \leq F(X) \leq l\tilde{F}(X)$. \looseness-1
 \end{lemma}
 \begin{proof}
We first prove the second part.  Its easy to see that $F(X) = \max_i f_i(X) \leq \sum_i f_i(X) = l\tilde{F}(X)$ which is the second inequality. The first inequality follows from the fact that $\sum_i f_i(X) \leq l \max_i f_i(X)$. To prove the first result, we start with proving $F(X) \leq \hat{F}(X)$. For a given set $X$, let $i_X$ be the index which maximizes $F$ so $F(X) = \sum_{j \in X} f_{i_X}(j)$. Then $f_{i_X}(j) \leq \max_i f_i(j)$ from which we get the result. Next, observe that $\hat{F}(X) \leq \sum_{i = 1}^l f_i(X) = l\tilde{F}(X) \leq lF(X)$ which follows from the inequality corresponding to $\tilde{F}$. Hence proved.
 \end{proof}
Finally, we show the impact of using the two modular approximations $\hat{F}(X)$ and $\tilde{F}(X)$ in the place of the function $F$ in \textsc{Robust-scsc}  and \textsc{Robust-scsk}. Lets call these Problems 1$^{\prime}$ and 2$^{\prime}$.
\begin{lemma}\label{mod-approx-bounds}
 Given $(\alpha, \beta)$ and $(\beta, \alpha)$ bicriteria approximation algorithms for Problems 1$^{\prime}$ and 2$^{\prime}$, we can obtain $(l\alpha, \beta)$ and $(\beta, l\alpha)$ bicriteria approximations for \textsc{Robust-scsc}  and \textsc{Robust-scsk}.
\end{lemma}
\begin{proof}
Lets start with $\tilde{F}$ (a very similar argument will also show this bound for $\hat{F}$). Using $\tilde{F}$ in 	\textsc{Robust-scsc} , we obtain $\min\{\tilde{F}(X) | G(X) \geq 1\}$. We can achieve a set $\tilde{X}$ such that $\tilde{F}(\tilde{X}) \leq \alpha \tilde{F}(X^*)$ and $G(\tilde{X}) \geq \beta$ where $X^*$ is the optimal solution to 	\textsc{Robust-scsc}. Now note that $F(\tilde{X}) \leq \alpha \tilde{F}(X^*) \leq l\alpha F(X^*)$ which proves the bound. Now lets use $\tilde{F}$ in 	\textsc{Robust-scsk}. 	\textsc{Robust-scsk}  becomes $\max\{g(X) | \tilde{F}(X) \leq 1\}$. Let $X^*$ be the optimal solution for the Problem $\max\{G(X) | F(X) \leq 1\}$. Then since $\tilde{F}(X^*) \leq F(X^*) \leq 1$, $X^*$ also a feasible solution for 	\textsc{Robust-scsk}$^{\prime}$. Let $\tilde{X}$ be the $(\beta, \alpha)$ bicriteria approximation for 	\textsc{Robust-scsk}$^{\prime}$. Then observe that $G(\tilde{X}) \leq \beta G(X^*)$ since $X^*$ is a feasible solution for \textsc{Robust-scsk}$^{\prime}$. Furthermore, $F(\tilde{X}) \leq l \tilde{F}(\tilde{X}) \leq l \alpha$. Hence $\tilde{X}$ is a $(\beta, l\alpha)$-approximate solution.
\end{proof}
We now prove Theorem~\ref{mod_f_submod_g_sec-theorem-complete}. 
\begin{proof}
\textbf{Case I, $k = 1, l \geq 1$: } Note that when $l \geq 1, k = 1$, 	\textsc{Robust-scsk}  reduces to Submodular Maximization subject to multiple knapsack constraints~\cite{kulik2009maximizing} and we can use the continuous greedy algorithm along with the rounding technique from~\cite{kulik2009maximizing}. Following Theorem 1 from \cite{kulik2009maximizing}, we get the $1 - 1/e - \epsilon$ single factor approximation algorithm for \textsc{Robust-scsk}. We can obtain the $(1 - \epsilon, 1 - 1/e - \epsilon)$ bi-criteria approximation for 	\textsc{Robust-scsc}  by combining this with Algorithm 1. Unfortunately, the continuous greedy is not easy to implement and can be very slow in practice (the complexity of the algorithm is $O(d/\epsilon^4 p(n))$ where $p(n)$ is a polynomial in $n$~\cite{kulik2009maximizing}). To remedy this, we can approximate the function $\max_i f_i(X)$ using Lemma~\ref{mod-approximations} into two modular approximations. This converts the multiple knapsack constraints into a single knapsack constraints which we can run twice (using the two modular approximations) and pick the better ones. There exists a discrete greedy algorithm with a higher order polynomial complexity~\cite{sviridenko2004note} (roughly a $O(n^4)$ algorithm) with a $1 - 1/e$ approximation and a much faster greedy algorithm with $O(n\log n)$ average time complexity~\cite{krause05note} which admits a $1/2(1 - 1/e)$ approximation. We can then use Lemma~\ref{mod-approx-bounds} from which we add an extra factor of $l$. Next, we show the $l(1 + \delta_g)$ single factor approximation for \textsc{Robust-scsc}. Again we use the modular approximations $\hat{F}(X)$ and $\tilde{F}(X)$ in the place of $F$, which reduces Robust Submodular Set Cover (	\textsc{Robust-scsc} ) into an instance of submodular set cover. In particular, the resulting algorithm first computes the two modular approximations for $F$ and then solves a submodular set cover (\textsc{ssc}) problem~\cite{wolsey1982analysis}. Finally, we pick the better amongst the two objectives. The greedy algorithm for \textsc{ssc} implies a guarantee of $1 + \delta_g$ and we can combine this with Lemma~\ref{mod-approx-bounds} to get the desired bounds. Finally, consider when $g(X)$ is modular, 	\textsc{Robust-scsk}  is a multiply constrained knapsack problem for which there exists fully polynomial time approximation schemes~\cite{pisinger1995algorithms,zhang1999approximation}. Thanks to the binary search reduction (Algorithm 2), we can achieve a bi-criteria FPTAS for 	\textsc{Robust-scsc}. We can similarly achieve a single factor approximation for 	\textsc{Robust-scsc}  by first, using the Modular approximation for $F$ which turns 	\textsc{Robust-scsc}$^{\prime}$ into a knapsack problem. Since the LP relaxation provides a $2$-approximation solution to 	\textsc{Robust-scsc}$^{\prime}$, we can combine this with Lemma~\ref{mod-approx-bounds} which gives the $2l$ approximation.

\textbf{Case II, $k \geq 1, l = 1$: } In this case, we can use the \emph{saturate} trick from~\cite{krause08robust} and use a proxy function $\hat{G}(X) = \sum_{i = 1}^k \min(g_i(X), 1)$. In particular, note that the constraint $G(X) = \min_{i = 1:k} g_i(X) \geq 1$ is equivalent to $\hat{G}(X) \geq l$. Hence 	\textsc{Robust-scsc}  is exactly an instance of submodular set cover (since $f$ is modular), and we achieve the approximation factor of $1 + \delta^k_g$. Note also that 	\textsc{Robust-scsk}  is a robust submodular maximization subject to a single knapsack constraint, and we can use Algorithm 1 to transform this into a bi-criteria factor (note this is actually very similar to the algorithm of \cite{krause08robust}). We can also directly solve 	\textsc{Robust-scsk}  using the algorithm from \cite{anari2019robust} (see Algorithm 3 in \cite{anari2019robust}). The authors attempt to directly solve the robust problem rather than relying on the binary search transformation. The resulting algorithm which is a greedy algorithm, can be much faster than the corresponding binary search procedure and achieves a factor of $(1 - \epsilon, \log k/\epsilon)$ (see Corollary 1 in~\cite{anari2019robust}). This factor gets rid of the dependence on $g$ but adds a somewhat worse dependence on $\epsilon$. Finally, we note that the bounds are the same even when the functions $g_i$'s are modular.

\textbf{Case III, $k \geq 1, l \geq 1$: } Next, we study the case with $l > 1, k > 1$, i.e. the most general case (equations~\eqref{eqref1} and \eqref{eqref2}. There are three main algorithms here. The first algorithm directly uses the approximations $\tilde{F}$ and $\hat{F}$ on $F$. 	\textsc{Robust-scsc}  then becomes an instance of submodular set cover after the \emph{saturate} trick from above, we get the approximation factor of $l(1 + \delta_g)$. The second algorithm is the modified greedy algorithm from~\cite{anari2019robust} which we directly apply to two modular approximations of $F$ (and picks the better one). Since we obtain a factor of $(1 - \epsilon, \log k/\epsilon)$ with a single knapsack, we get a cumulative bound of $(1 - \epsilon, O(l.\log \frac{k}{\epsilon}))$ which takes into account the fact that our modular approximation to $F$ is itself a $l$-approximation. While this is a worse approximation factor, the algorithm is a discrete greedy algorithm~\cite{anari2019robust} and is scalable. Finally, we can use the continuous greedy algorithm for \textsc{Robust-SubMax} with multiple knapsack constraints (from Theorem 9). This will enable us to get rid of the additional $l$ factor.
\end{proof}
Similar to the previous case, the same bounds hold when $g_i$'s are modular.

\subsection{Submodular $f_i$'s, Submodular (or Modular) $g_i$'s}
In this section, we shall assume for simplicity that all the submodular functions $f_i$'s have the same curvature $\kappa_f$. The results can be trivially be extended even when the curvatures are not the same. In this section, we study the following main classes of algorithms. 
\begin{itemize}
    \item \textbf{Majorization-Minimization: } The first is Majorization Minimization (\textsc{mmin}). \textsc{mmin} proceeds as follows. Initialize $X^0 = \emptyset$. Then obtain $X^{t+1}$ in round $t+1$ by solving \textsc{Robust-scsc}  and \textsc{Robust-scsk} , but using the upper bounds $m^{f_i}_{X^t}(X)$ in place of the functions $f_i$. We proceed only if we make sufficient progress. Once we convert this into a problem with modular $f_i$'s, we can use the Algorithms from Section 5.3.
    \item \textbf{Ellipsoidal Approximation: } The second algorithm is the Ellipsoidal Approximation (\textsc{ea}). Recall that any polymatroid (monotone submodular) function $f$, can be approximated by a function of the form $\sqrt{w^f(X)}$ for a certain modular weight vector $w^f \in \mathbb R^V$. The idea of the \textsc{ea} is to approximate $f_i$'s by the corresponding Ellipsoidal Approximations. Once we convert this into a problem with modular $f_i$'s, we can use the Algorithms from Section 5.3.
    \item \textbf{Average Approximations of \textsc{mmin} and \textsc{ea}: } Another framework of algorithms is to approximate the function $F(X) = \max_i f_i(X)$ by the \emph{average approximation} $\tilde{F}(X) = \frac{\sum_{i = 1}^l f_i(X)}{l}$ (very similar to what we did in the case of \textsc{Robust-SubMin}. Note that $\tilde{F}$ is a submodular function and we can replace $F$ in \textsc{Robust-scsc}  and \textsc{Robust-scsk}  with $\tilde{F}$. Note that we immediately get back the special case of $l = 1$. We can similarly use \textsc{mmin} and \textsc{ea} and we call the resulting algorithms \textsc{mmin-aa} and \textsc{ea-aa}. Once we convert this into a problem with modular $f_i$'s, we can use the Algorithms from Section 5.3.
\end{itemize}
\begin{theorem}\label{thm:gen_f_gen_g}
The following approximation guarantees hold for Submodular functions $f_i$'s and $g_i$'s. 
\begin{enumerate}
    \item Case I, $l = 1, k \geq 1$: \textsc{mmin} achieves a $(1 + \delta^k_g)K(n, \kappa_f)$ approximation for 	\textsc{Robust-scsc}  and $(1 - \epsilon, K(n, \kappa_f)\log k/\epsilon)$ bicriteria approximation for 	\textsc{Robust-scsk}. \textsc{ea} achieves an approximation factor of $O(K(\sqrt{n}\log n, \kappa_f)[1 + \delta_g])$ for 	\textsc{Robust-scsc}  and \\ $(1 - \epsilon, O(K(\sqrt{n}\log n, \kappa_f)[1 + \delta_g]))$ for 	\textsc{Robust-scsk}. 
    \item Case II, $l \geq 1, k \geq 1$: \textsc{mmin-aa} achieves a $l(1 + \delta^k_g)K(n, \kappa_f)$ approximation for 	\textsc{Robust-scsc}  and a bicriteria approximation of $(1 - \epsilon, lK(n, \kappa_f)\log k/\epsilon)$ for 	\textsc{Robust-scsk}. Similarly \textsc{ea-aa} achieves an approximation factor of $O(lK(\sqrt{n}\log n, \kappa_f)[1 + \delta^k_g])$ for 	\textsc{Robust-scsc}  and $(1 - \epsilon, O(lK(\sqrt{n}\log n, \kappa_f)[1 + \delta^k_g]))$ for 	\textsc{Robust-scsk}. \textsc{mmin} achieves a bicriteria approximation of $(K(n, \kappa_f)\log{k/\epsilon}(1 + \epsilon), 1 - \epsilon)$ for 	\textsc{Robust-scsc}  and a factor of $(1 - \epsilon, K(n, \kappa_f)\log{k/\epsilon})$ for 	\textsc{Robust-scsk}. Similarly, \textsc{ea} achieves a bicriteria approximation of $(O(\sqrt{n}\log n \sqrt{\log{k/\epsilon}}), 1 - \epsilon)$ for 	\textsc{Robust-scsc}  and a factor of $(1 - \epsilon, O(\sqrt{n}\log n \sqrt{\log{k/\epsilon}}))$ for 	\textsc{Robust-scsk}. For the special case of $l \geq 1, k = 1$, the parameters $\delta^k_g$ are replaced by $\delta_g$ and $\log k/\epsilon$ by $1 - 1/e$ wherever applicable. The rest of the results remain the same.
\end{enumerate}
\end{theorem}
Recall from Lemmata~\ref{curvaturemmin} and \ref{curvatureea} that $m^f_{\emptyset}$ and $\hat{f}$ are bounded approximations of the function $f$. The main idea of \textsc{mmin} and \textsc{ea} are to use these approximations. We now prove Theorem~\ref{thm:gen_f_gen_g}.

\begin{proof}
\textbf{Case I, $l = 1, k \geq 1$: } Lets start with \textsc{mmin}. Since we iteratively approximate $f$ with a Modular upper bound, we then iteratively solve \textsc{Robust-scsc}  and \textsc{Robust-scsk}  with $f$ being Modular (covered in Section 5.3, Case II). We show the result for the first iteration of \textsc{mmin}. Recall we start with the empty set $X^0 = \emptyset$. Following Theorem~\ref{mod_f_submod_g_sec-theorem-complete} (Case II), we can achieve a set $\hat{X}$ which satisfies $m^f_{\emptyset}(\hat{X}) \leq (1 + \delta^k_g)m^f_{\emptyset}(X^*)$. This implies the following sequence of inequalities: $f(\hat{X}) \leq m^f_{\emptyset}(\hat{X}) \leq (1 + \delta_g)m^f_{\emptyset}(X^*) \leq (1 + \delta_g)K(n, \kappa_f) f(X^*)$. Finally, since we assume that make progress in every iteration, the objective value at convergence will be no worse than this factor. We can similarly show the bound for Problem 4 which we skip in the interest of space. Next, consider the \textsc{ea} algorithm. We replace $f$ with its curve-normalized approximation $\hat{f}$. However since $\hat{f}$ is not modular any more, it is not obvious how we can optimize this over with the submodular set cover constraints (note that the we can replace the constraints in 	\textsc{Robust-scsc}  with a single submodular set cover constraint using the \emph{saturate} trick). Moreover, we can use the fact that this function is of the form $\sqrt{w_1(X)} + w_2(X)$ and can be optimized up to a fully polynomial approximation scheme~\cite{nikolova2010approximation,nipssubcons2013}. In particular, $\hat{f}$ can be optimized over a submodular set cover constraint up to a factor of $(1 + \delta_g)(1 + \epsilon)$~\cite{nipssubcons2013}. Combining this with the fact that $\hat{f}(X) \leq f(X) \leq O(K(\sqrt{n}\log n, \kappa_f)) \hat{f}(X)$, we get the required bound.

\textbf{Case II, $l \geq 1, k \geq 1$: } Next, we shall look at the case $k \geq 1$ and $l \geq 1$, and also highlight some minor differences in the results for $k = 1$. The \textsc{aa} algorithms simply approximate $F(X) = \max_i f_i(X)$ with the average approximations $\tilde{F}$. Combining the bounds from Case I and Lemma~\ref{mod-approx-bounds}, we get the results.  While the \textsc{aa} framework obtains tight bounds based on the curvature and $n$, there is an additional dependence on $l$. We can get rid of this by directly using \textsc{mmin} and \textsc{ea} on \textsc{Robust-scsc}  and \textsc{Robust-scsk}. We first study \textsc{mmin}. Since $f_i$'s are replaced by $m^{f_i}$'s which are modular functions, we can use Theorem~\ref{mod_f_submod_g_sec-theorem-complete} (Case III) and combine it with the approximation bound for $m^{f_i}$. With \textsc{ea}, it is no longer straight-forward to optimize the curve normalized function. Instead, we use the basic \textsc{ea}: $\sqrt{w_{f_i}(X)}$. Lets start with 	\textsc{Robust-scsk}. The constraints here are $\sqrt{w_f^i(X)} \leq b_i$ which is equivalent to $w_f^i(X) \leq b_i^2$. We then have \textsc{Robust-SubMax} subject to multiple knapsack constraints and we can use Theorem~\ref{mod_f_submod_g_sec-theorem-complete} (case III). Similarly we can take the square of the objective function in 	\textsc{Robust-scsc}  and we again obtain 	\textsc{Robust-scsc}  with Modular $f_i$'s. Also note here that the same analysis goes through when we have $k = 1$, i.e. a single function $g$. In this case, we use the results from Case I instead of Case III in Section 5.3, and we essentially need to replace the $\log k/\epsilon$ by $1 - 1/e$ and $\delta^k_g$ by $\delta_g$.
\end{proof}

\begin{figure}[h]
    \centering
    \includegraphics[width = 0.40\textwidth]{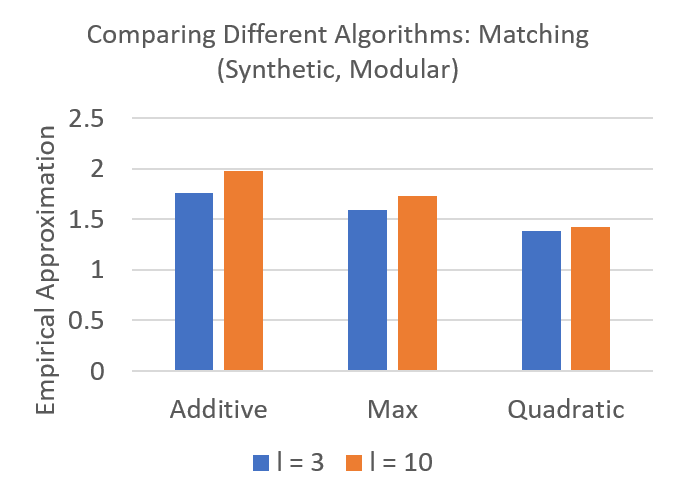} 
    \includegraphics[width = 0.49\textwidth]{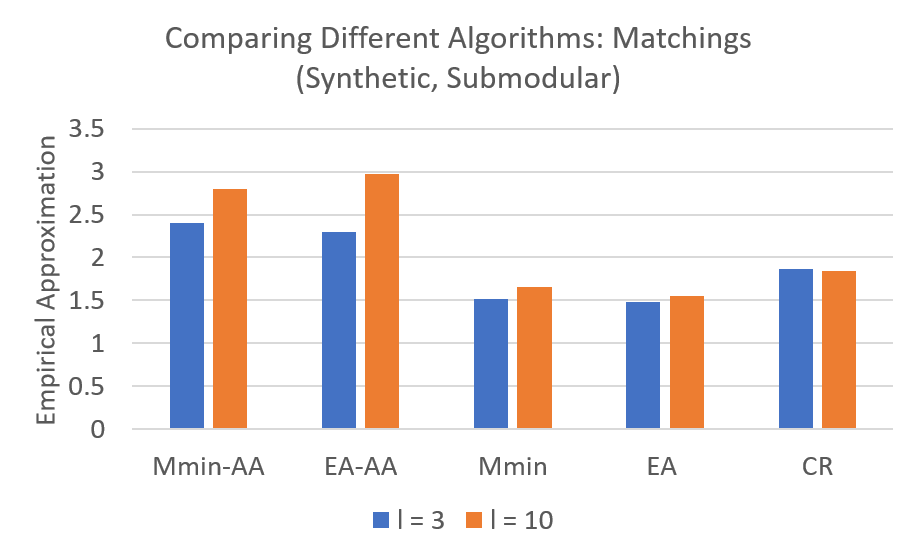}   
    \includegraphics[width = 0.51\textwidth]{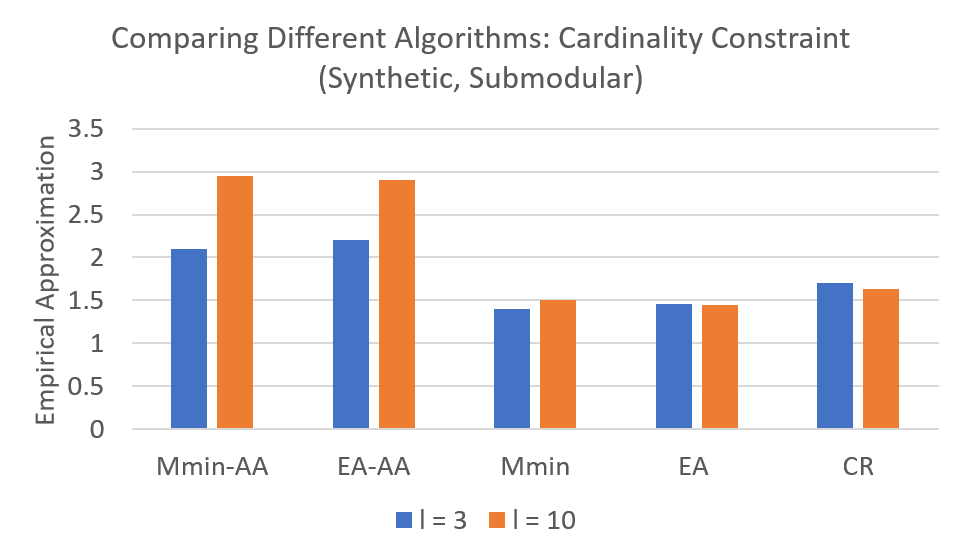}
    \caption{Synthetic experiments with top left: Robust Modular Minimization subject to matching constraints, top right: \textsc{Robust-SubMin} with matching constraints and bottom: \textsc{Robust-SubMin} with cardinality constraints.}
    \label{fig:synthetic}
\end{figure}

\section{Experimental Results}
\subsection{Synthetic Experiments} 
The first set of experiments are synthetic experiments. We define $f_j(X) = \sum_{i = 1}^{|\mathcal C_j|} \sqrt{w(X \cap C_{ij})}$ for a clustering $\mathcal C_j = \{C_1, C_2, \cdots, C_{|\mathcal C_j|}$. We define $l$ different random clusterings (with $l = 3$ and $l = 10$). We choose the vector $w$ at random with $w \in [0, 1]^n$. We compare the different algorithms under cardinality constraints $(|X| \leq 10)$ and matching constraints. For cardinality constraints, we set $n = 50$. For matchings, we define a fully connected bipartite graph with $m = 7$ nodes on each side and correspondingly $n = 49$. The results are over 20 runs of random choices in $w$ and $\mathcal C$'s and shown in Figure 2 (top row). The first plot compares the different algorithms with a modular function (the basic min-max combinatorial problem). The second and third plot on the top row compare the different algorithms with a submodular objective function defined above. In the submodular setting, we compare the average approximation baselines (\textsc{mmin-aa}, \textsc{ea-aa}), Majorization-Minimization (\textsc{mmin}), Ellipsoidal Approximation (\textsc{ea}) and Continuous Relaxation (CR). For the Modular cases, we compare the simple additive approximation of the worst case function, the Max-approximation and the quadratic approximation.  We use the graduated assignment algorithm~\cite{gold1996graduated} for the quadratic approximation to solve the quadratic assignment problem. For other constraints, we can use the efficient algorithms from~\cite{buchheim2018quadratic}

Figure 2 shows the results in the modular setting.  First, as expected (Figure 2, top left), we see that in the modular setting, the simple additive approximation of the max function doesn't perform well and the quadratic approximation approach performs the best. Figure 2 (top right and bottom) show the results with the submodular function under matching and cardinality constraint. Since the quadratic approximation performs the best, we use this in the \textsc{mmin} and \textsc{ea} algorithms. First we see that the average approximations (\textsc{mmin-aa} and \textsc{ea-aa}) don't perform well since it optimizes the average case instead of the worst case. Directly optimizing the robust function performs much better. Next, we observe that \textsc{mmin} performs comparably to \textsc{ea} though its a simpler algorithm -- a fact which has been noticed in several other scenarios as well~\cite{rkiyersemiframework2013,nipssubcons2013,jegelka2011-inference-gen-graph-cuts}

\begin{figure}[h]
    \includegraphics[width = 0.50\textwidth]{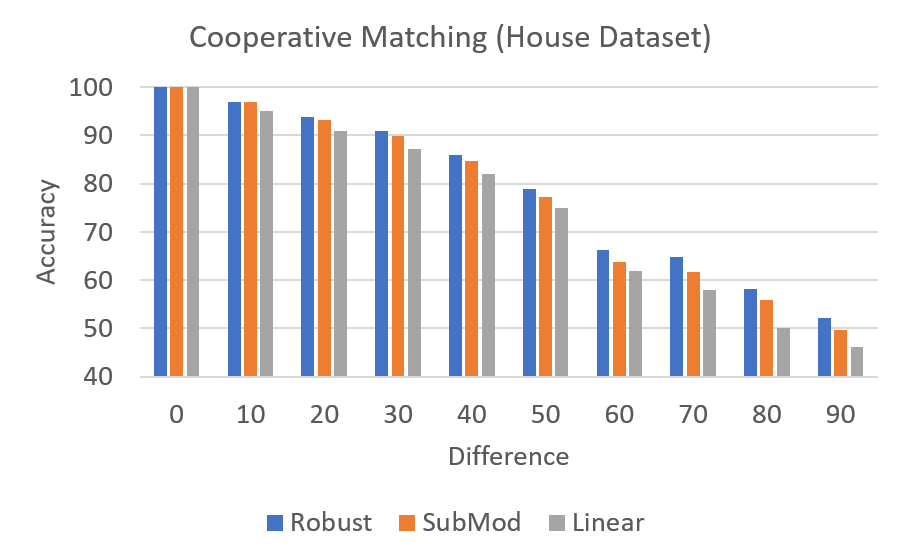}
    \includegraphics[width = 0.50\textwidth]{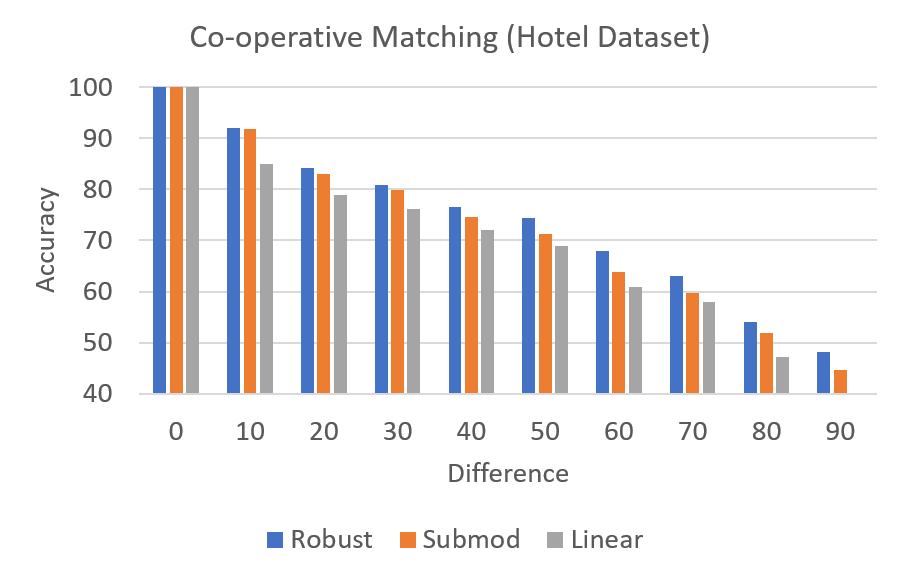}
    \caption{Experiments on Co-operative Matching with left: House and right: Hotel datasets. We see that \textsc{Robust-SubMin} consistently outperforms \textsc{SubMod} and \textsc{Linear} Assignment baselines~\cite{iyer2019near}.}
    \label{fig:my_label}
\end{figure}

\subsection{Co-operative Matchings} 
In this set of experiments, we compare \textsc{Robust-SubMin} in co-operative matchings. We follow the experimental setup in~\cite{iyer2019near}, which we describe here for completeness. We begin by describing the choice of the submodular functions. We cluster the keypoints separately for each of the two images into $k = 3$ groups. The clustering can be performed based on
the pixel color map, or simply the distance of the key-points. That is,
each image has $k$ clusters. Let $\{ V_i^{(1)} \}_{i=1}^k$ and $\{
V_i^{(2)} \}_{i=1}^k$ be the two sets of clusters.  We then compute
the linear assignment problem, letting $\mathcal M \subseteq \mathcal
E$ be the resulting maximum matching. We then partition the edge set $\mathcal
E = \mathcal E_1 \cup \mathcal E_2 \cup \dots \mathcal E_k \cup
\mathcal E'$ where $\mathcal E_i = \mathcal M \cap (V_\ell^{(1)}
\times V_s^{(2)} )$ for $\ell,s \in \{ 1, 2, \dots, k\}$ corresponding
to the $i$'th largest intersection, and $\mathcal E' = \{\mathcal E
\backslash \cup_{i = 1}^k \mathcal E_i\}$ are the remaining edges
either that were not matched or that did not lie within a frequently
associated pair of image key-point clusters. We then
define a submodular function as follows:
\begin{align}\label{coopobj}
f(S) = \sum_{i = 1}^k \psi_i(w(S \cap \mathcal E_i)) + 
w(S \cap \mathcal E'), 
\end{align}
which provides an additional discount to the edges $\{\mathcal
E_i\}_{i = 1}^k$ corresponding to key-points that were frequently
associated in the initial pass. The problem of co-operative matching then becomes an instance of constrained submodular minimization with the submodular function (over the edges) defined above, and a constraint that the edges form a matching. To instantiate \textsc{Robust-SubMin}, we construct $l$ clusterings of the pixels: $\{(\mathcal E^1_1, \cdots, \mathcal E^1_k), \cdots, (\mathcal E^l_1, \cdots, \mathcal E^l_k)\}$ and define a robust objective:
\begin{align}
    f_{robust}(S) = \max_{i = 1:l} \sum_{j = 1}^k \psi_j(w(S \cap \mathcal E^i_j)) + 
w(S \cap \mathcal E^i{'})
\end{align}

We run this on the House and Hotel Datasets~\cite{caetano2009learning}. The results in Figure 3.The house dataset has $111$ images, while the hotel dataset has $101$ images. We consider all possible pairs of images, with differences between the two images ranging from $0:10:90$ in both cases. We consider three algorithms. The first is simple Linear assignment (i.e. modular function), the second is a single submodular function minimization subject to matching constraint (\textsc{SubMod}), and the third is \textsc{Robust-SubMin} subject to a matching constraint. We use the PLA algorithm from~\cite{iyer2019near} for \textsc{SubMod}, since that was observed to outperform \textsc{mmin}. Note that \textsc{SubMod} just uses a single clustering. For \textsc{Robust-SubMin}, we construct $l = 10$ clusterings. We observe that $l = 10$ tends to perform the best, since with too many clusterings, the algorithms become both slower and also perform worse. The different clustering each obtained by different random initializations of k-means. From the results (Figure 3), we see that \textsc{Robust-SubMin} consistently outperforms both \textsc{Linear} and \textsc{Sub-Min} by about 2 - 5\% in accuracy. 

\begin{figure}[h]
    \centering
        \includegraphics[width = 0.5\textwidth]{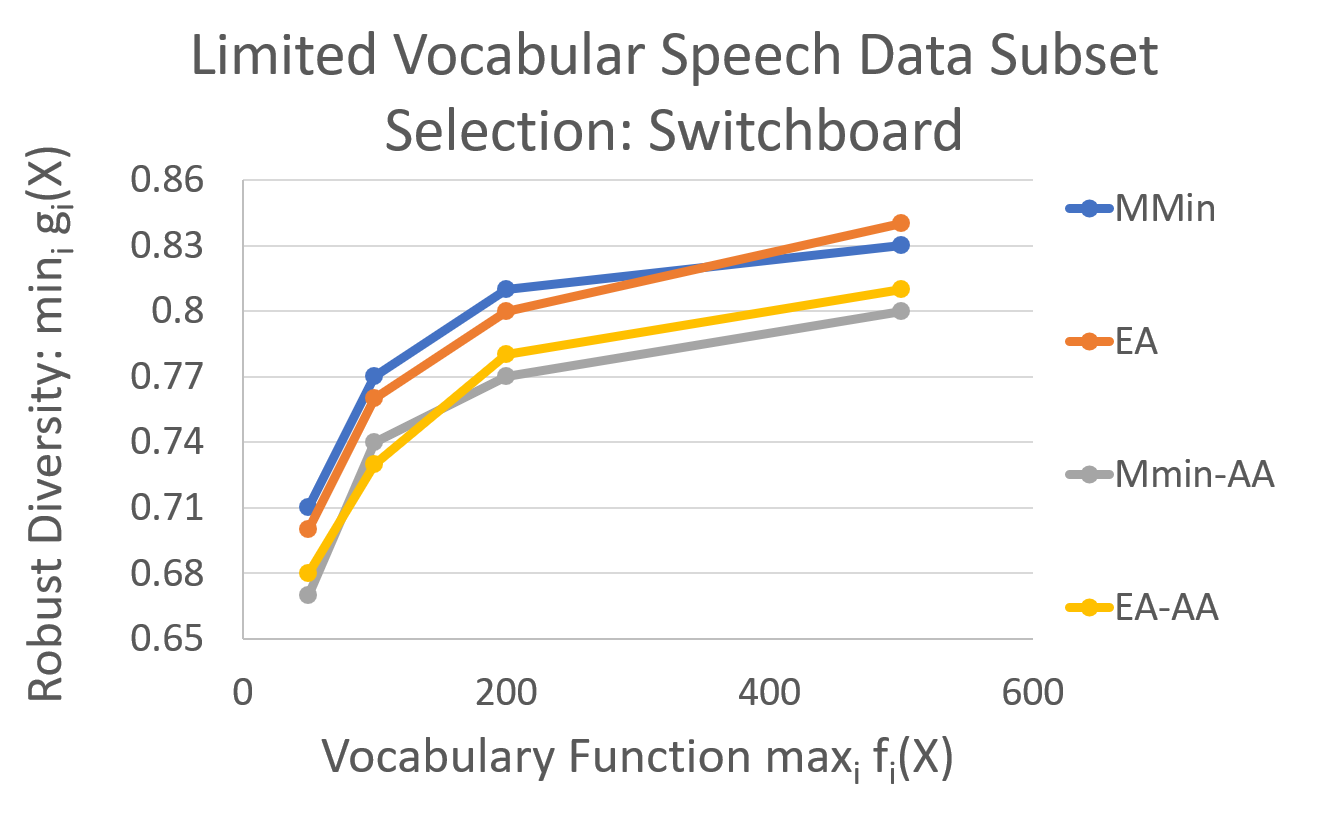}
    \caption{Data Subset Selection Results}
    \label{fig:my_label}
\end{figure}

\subsection{Robust Data Subset Selection} 
In this section, we highlight the scalability and efficacy of our algorithms for \textsc{Robust-Scsc} and \textsc{Robust-Scsk} on robust data subset selection with limited vocabulary.  We follow the experimental setup from~\cite{liu2017svitchboard}. The function $f$ is the vocabulary function, $f(X) = w(\gamma(X))$~\cite{liu2017svitchboard}. In our experiments, we define several different functions $f_i$ (with $l = 10$). To define each function, we randomly delete 20\% of the words from the vocabulary and define $f_i(X) = w(\gamma_i(X))$. The function $g$ is the diversity function and we use the feature based function from~\cite{liu2017svitchboard}. Similar to the vocabulary functions, we define $k = 10$ different functions $g_i$ each defined via small perturbations of the feature values. The functions $g_i$ are normalized so they are all between $[0, 1]$. The intuition of selecting $f_i$'s and $g_i$'s in this way is to be robust to perturbations in features and vocabulary words. These then become instances of \textsc{Robust-Scsc} or \textsc{Robust-Scsk}. In this experiment, we apply the $f_i$'s as constraints, and as a result, we use the formulation of \textsc{Robust-Scsk}. We run these experiments on Switchboard-I~\cite{godfrey1992switchboard} and we restrict ourselves to a ground-set size of $|V| = 100$. We compare the four algorithms: \textsc{mmin-aa}, \textsc{ea-aa}, \textsc{mmin}, and \textsc{mmin}. The results are shown in Figure 4. As expected, \textsc{mmin-aa} and \textsc{ea-aa} don't perform as well because they just average the $\max$ with the average so do not model the \emph{robustness}. Furthermore, though \textsc{ea} achieves the best theoretical approximation factors, it performs comparably to \textsc{mmin} empirically, though \textsc{mmin} is much faster in practice. \textsc{mmin} runs in around 1.2 seconds while \textsc{ea} takes around 3 minutes. For larger data-sets, the difference is seconds versus hours. 

\begin{figure}
    \centering
    \includegraphics[width = 0.45\textwidth]{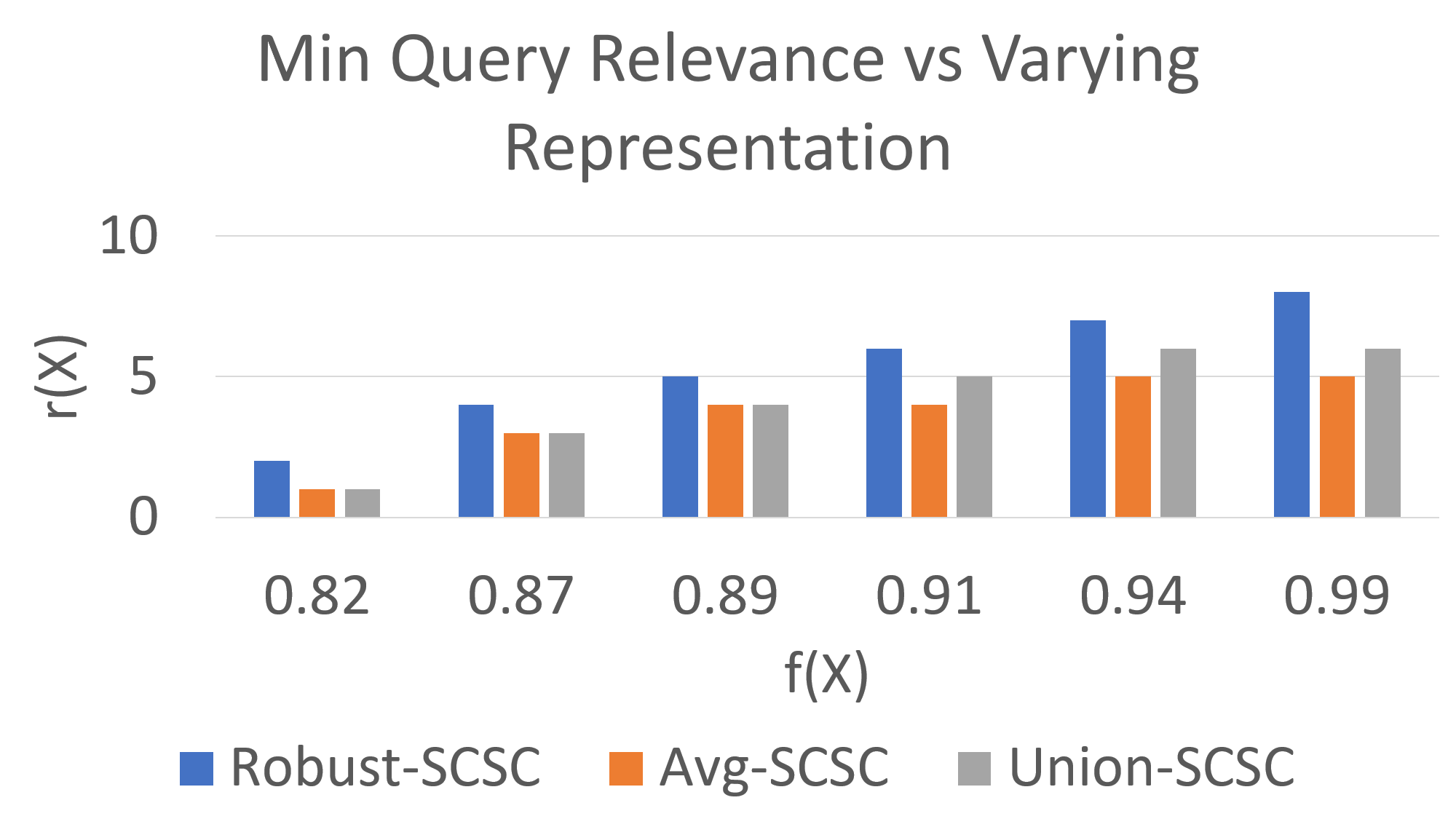}
    \includegraphics[width = 0.45\textwidth]{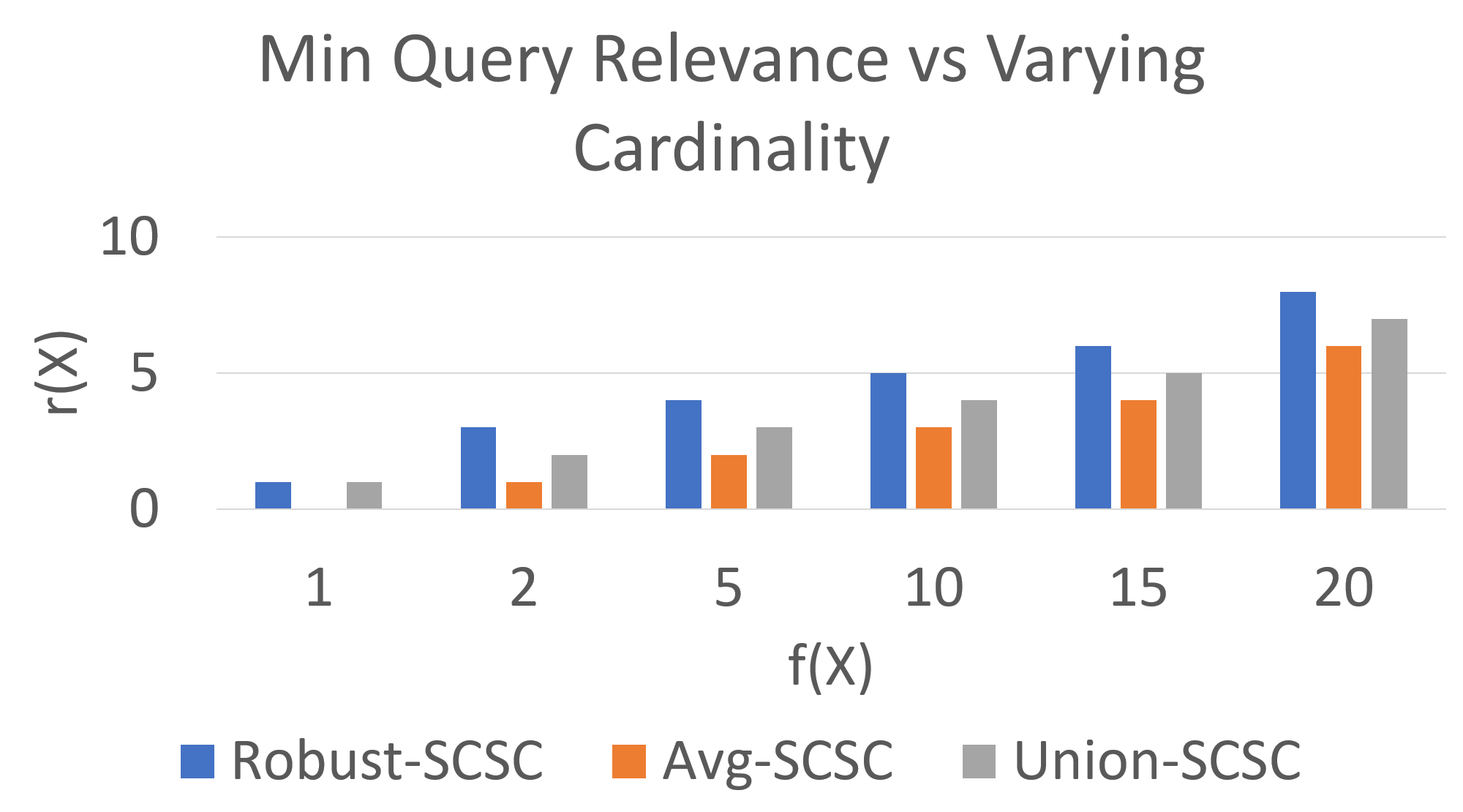}
    \vspace{4mm}
    
    \includegraphics[width = 0.5\textwidth]{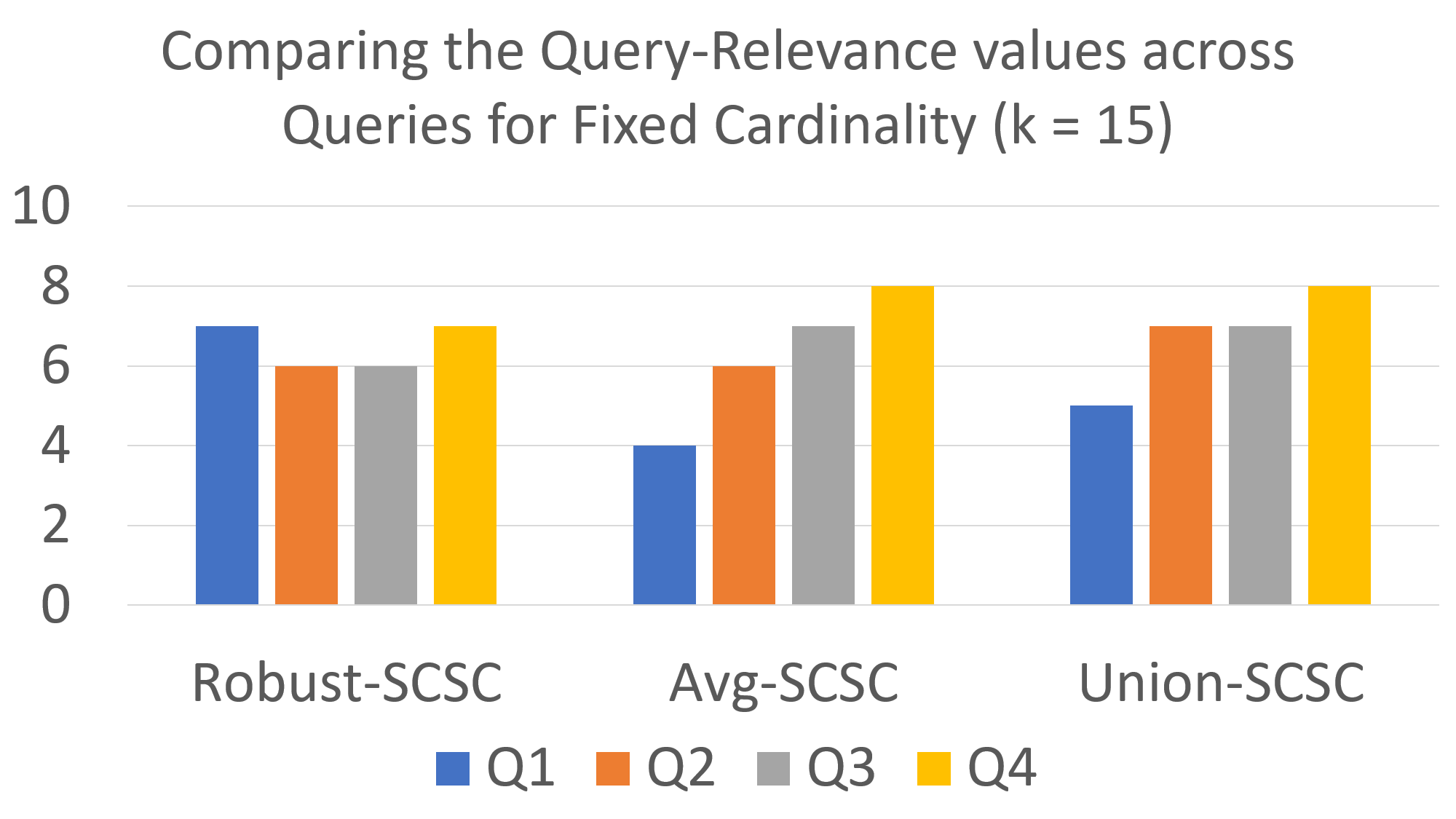}
    \caption{Comparing the performance of \textsc{Robust-Scsc} to baselines (\textsc{Avg-Scsc} and \textsc{Union-Scsc}. top left figure compares the minimum query relevance $\min_i R(X, Q_i)$ for different coverage values of $g(X)$, top right figure compares minimum query relevance $\min_i R(X, Q_i)$ for different values of cardinality (subset size), and finally, the bottom figure shows that the robust version is fair across the different queries compared to the Avg and Union versions, and tries to have the summary cover the queries \emph{equally}.}
    \label{fig:my_label}
\end{figure}
\subsection{Image Collection Summarization with Multiple Query Sets}
Next, we study the performance of \textsc{Robust-Scsk} in the real world application of image collection summarization. In particular, we study the problem of maximizing a single submodular function $g$ while minimizing the query-conditional functions $h(X | Q_i)$ for all queries $Q_i$. To evaluate the algorithms, we use the image collection summarization dataset from~\cite{tsciatchek14image}. The dataset has
14 image collections with 100 images each and provides many human summaries per collection. We use the approach of~\cite{kaushal2020unified} to consider query focused summarization. In particular, \cite{kaushal2020unified}
extend the dataset of \cite{tsciatchek14image} by creating dense noun concept annotations for every image. They start by designing the universe of concepts based on the 600 object classes in
OpenImagesv6~\cite{kuznetsova2020open} and 365 scenes in Places365~\cite{zhou2014learning}. They eliminate concepts common to both (for
example, closet) to get a unified list of 959 concepts. The authors then use pretrained YOLOv3~\cite{redmon2018yolov3} and ResNet50~\cite{he2016deep} models for objects and scenes and get initial annotations, which are then cleaned by human annotators. 

We use this dataset and create a set of ten query sets, with each query set containing four concepts each. Examples of these concepts are: "Person, Indoor, Outdoor, Forest, Trees, Beach, Bench, Building" etc. We represent each image using the probabilistic
feature vector taken from the output layer of YOLOv3 model pre-trained on OpenImagesv6 and concatenate it with the probability vector of scenes from the output layer of ResNet50 trained on
Places365 dataset. 

In our experiments, we use two choices for $g$. For $g$, we use either the cardinality or the facility location function, while we use the facility location function as the choice for $h$. As discussed in the motivating section, we pose the problem as minimizing the maximum query conditional function, while setting a coverage constraint on $g$. In other words, we solve: 
\begin{align}
    \min_{X \subseteq V} \max_{i = 1:l} h(X | Q_i) \,\, | \,\, g(X) \geq c
\end{align}
Since in our experiments $g$ is monotone, as we vary $c$, we vary the size of the obtained subset. We compare \textsc{Robust-Scsc} (i.e. the optimization problem above) with two baselines. The first is where we just take the sum, i.e. minimize a single average function $\sum_{i = 1}^l h(X | Q_i)$ and pose this as \textsc{Scsc}: $\min_{X \subseteq V} \sum{i = 1}^l h(X | Q_i) \,\, | \,\, g(X) \geq c$. Note that this is very similar to the Average-Approximation scheme and we call this \textsc{Avg-Scsc}. The second baseline is we simply take a union over all the the query concepts and use the function $h(X | \cup_i Q_i)$. In this case, the optimization problem is: $\min_{X \subseteq V}h(X | \cup_i Q_i) \,\, | \,\, g(X) \geq c$. We call this \textsc{Union-Scsc}. For the facility location function, we compute a similarity kernel over the feature vector constructed above. To evaluate the algorithms, we plot the tradeoff between the values of $g$ (which is either the subset size or the representation) and the query relevance. We measure query relevance for each query $Q_i$ by finding the number of images in $X$ which are relevant to $Q_i$ (we denote this by $r(X, Q_i)$). We then plot $g(X)$ on the y-axis and $r(X) = \min_i r(X, Q_i)$ on the x-axis as we vary the subset size. 

Since we know the \textsc{mmin} performs comparable to \textsc{ea}, we use \textsc{mmin} as the algorithm here due to its scalability. Similarly for the baselines (\textsc{Avg-Scsc} and \textsc{Union-Scsc}), we use the \textsc{mmin} algorithms since both those problems are instances of \textsc{Scsc}~\cite{nipssubcons2013}. The results are shown in Figure 5. We run this on the first two datasets out of the 14 image collection sets. We see that the \textsc{Robust-Scsk} with \textsc{mmin} consistently outperforms the single query terms with both the sum and the union. This implies that considering all queries individually (via the robust or worst case formulation) does make sure we obtain a set which is relevant to \emph{all} the queries and is therefore fair. We see that \textsc{Robust-Scsc} has a higher value of the minimum query relevance $\min_i r(X, Q_i)$ across the queries compared to \textsc{Union-Scsc} and \textsc{Avg-Scsc}. This is even more clear in the bottom plot in figure 5, where we see that because of the min-max formulation, \textsc{Robust-Scsc} ensures that all of the four queries are equally represented in the subset $X$, in contrast to the Avg and Union versions.

\section{Conclusions}
In this paper, we study four formulations Robust Submodular Optimization (\textsc{Robust-SubMin}, \textsc{Robust-SubMax}, \textsc{Robust-Scsc} and \textsc{RobustScsk}) from a theoretical and empirical perspective. We study approximation algorithms and hardness results. We propose a scalable family of algorithms (including the majorization-minimization algorithm) and theoretically and empirically contrast their performance. We empirically compare and contrast the performance of our algorithms on three real world applications: \textsc{Robust-SubMin} with matching constraints for image correspondence, \textsc{Robust-Scsk} for data subset selection with vocabulary constraints, and \textsc{Robust-Scsc} for image collection summarization with multiple queries. In each case, we show that, a) the designed algorithms, and particularly, the majorization-minimization perform better than other heuristics like the average approximation, and also perform comparably to the much more complicated ellipsoidal approximation approach; and b) the robust formulation in fact outperforms the single submodular function based non-robust approaches. In future work, we would like to address the gap between the hardness and approximation bounds, and achieve tight curvature-based bounds in each case. We would also like to study other settings and formulations of robust optimization in future work.\looseness-1

\bibliographystyle{theapa}
\bibliography{ecai}
\end{document}